\documentclass[11pt,a4paper]{article}

\usepackage{a4wide}
\usepackage{graphicx}
\usepackage{latexsym}
\usepackage{epsfig}
\usepackage{amssymb}
\usepackage{amstext}
\usepackage{amsgen}
\usepackage{amsxtra}
\usepackage{amsgen}
\usepackage{amsthm,mathrsfs,mathtools}
\usepackage{bbm}

\usepackage{enumerate,amsmath,subfigure,color}

\newtheorem{thm}{Theorem}[section]
\newtheorem{prop}[thm]{Proposition}
\newtheorem{lemma}[thm]{Lemma}
\newtheorem{cor}[thm]{Corollary}
\newtheorem{rem}[thm]{Remark}

\theoremstyle{definition}

\newtheorem{assumption}[thm]{Assumption}

\newcommand{\N}{\mathbb{N}}

\newcommand{\R}{\mathbb{R}}

\newcommand{\E}{\mathbb{E}}
\newcommand{\Prob}{\mathbb{P}}

\newcommand{\LO}{\mathcal{L}}

\def\fdan{f^\delta_{\alpha,N}}
\def\sdiff{r}
\def\pdan{\sdiff^\delta_{\alpha,N}}
\def\Jdan{J^\delta_{\alpha,N}}
\def\gn{{\bf g}_N}
\def\gdan{{\bf g}^\delta_N}

\newcommand{\norm}[1]{\left\|#1 \right\|}

\newcommand{\nsG}[1]{\left\|#1 \right\|_{\text{sG}}}

\def\<{\langle}
\def\>{\rangle}

\DeclareMathOperator*{\argmin}{arg\,min}

\newcommand{\Npxl}{N_{\text{pxl}}}
\newcommand{\Ndtc}{N_{\text{dtc}}}
\newcommand{\Nth}{N_{\theta}}
\newcommand{\prox}{\ensuremath{\mathrm{prox}}} 
\newcommand{\sign}{\ensuremath{\mathrm{sign}}} 
\newcommand{\supp}{\ensuremath{\mathrm{supp}}} 

\renewcommand{\vec}[1]{\boldsymbol{#1}} 
\newcommand{\w}{\vec{w}} 
\newcommand{\f}{\vec{f}} 
\newcommand{\g}{\vec{g}} 
\newcommand{\Aop}{\vec{A}} 
\newcommand{\Wop}{\vec{W}} 
\newcommand{\x}{\vec{x}} 
\newcommand{\z}{\vec{z}} 
\newcommand{\thetab}{\vec{\theta}} 
\newcommand{\gNd}{\g_N^\delta} 
\newcommand{\faNd}{\f_{\alpha,N}^\delta} 

\def\bu{{\bf u}}
\def\Abu{A_{{\bf u}}}
\def\Bbu{B_{{\bf u}}}
\def\Sbu{S_{{\bf u}}}
\def\calR{\mathscr{R}}
\def\coef{c_{\lambda, p, s, d}}
\def\coefb{c_{\lambda, q, -s, d}}
\DeclareMathOperator{\tr}{Tr}

\definecolor{pinegreen}{rgb}{0,0.55,0.3}

\graphicspath{ {./images/} }

\begin{document}

\title{Convex regularization in statistical inverse learning problems}
\author{Tatiana A.~Bubba$^1$\footnote{T.A.~Bubba did most of this work while at the Department of Mathematics and Statistics, University of Helsinki, Finland.}, Martin Burger$^2$, Tapio Helin$^3$ and Luca Ratti$^4$\footnote{e-mails:  \texttt{tb715@cam.ac.uk}, \texttt{martin.burger@fau.de}, \texttt{tapio.helin@lut.fi}, \texttt{luca.ratti@unige.it}}}
\date{\small{%
$^1$Department of Applied Mathematics and Theoretical Physics, University of Cambridge. Wilberforce Road, CB3 0WA Cambridge, UK. \\
$^2$Department of Mathematics, Friedrich-Alexander Universit\"{a}t Erlangen-N\"{u}rnberg. Cauerstra{\ss}e 11, 91058 Erlangen, Germany. \\
$^3$ School of Engineering Science, Lappeenranta University of Technology. Yliopistokatu 34, 53850 Lappeenranta, Finland.\\
$^4$ MaLGa Center, Department of Mathematics, University of Genoa, Via Dodecaneso 35, 16146 Genova, Italy.
}}

\maketitle

\begin{abstract}
 We consider a statistical inverse learning problem, where the task is to estimate a function $f$ based on noisy point evaluations of $Af$, where $A$ is a linear operator. The function $Af$ is evaluated at i.i.d.~random design points $u_n$, $n=1,...,N$ generated by an unknown general probability distribution. We consider Tikhonov regularization with general convex and $p$-homogeneous penalty functionals and derive concentration rates of the regularized solution to the ground truth measured in the symmetric Bregman distance induced by the penalty functional. We derive concrete rates for Besov norm penalties and numerically demonstrate the correspondence with the observed rates in the context of X-ray tomography.
\end{abstract}
        


\section{Introduction}

Inverse problems study how indirect observational data can be processed into information about objects of interest in a robust manner.
The literature of inverse problems often adopts the perspective that the observational process can be designed or controlled to a sufficient degree. However, for many large-scale inverse problems in modern science and engineering massive data sets arise from poorly controllable experimental conditions. Such problems are closely connected to statistical learning setting, where the objective is to approximate a function $g : U \to V$ through a set of pairs $(u_n,v_n)_{n=1}^N$ drawn from an unknown probability measure on $U\times V$.

The framework, where observational data is limited to a finite set of random and noisy point evaluations of the output, has also a tradition in inverse problems \cite{o1990convergence, bissantz2004consistency}. In particular, statistical inverse learning problems have recently gained attention and we give an overview below. 
Our interest lies in deriving convergence rates, for general regularization schemes, of the expected reconstruction error, namely, the distance (in a suitable metric) between the solutions of the inverse problem and the regularized one. In this regard, the state of the art was recently improved by Blanchard and M\"{u}cke \cite{blanchard2018optimal}, who derive minimax optimal convergence rates for the general spectral regularization approach in Hilbert spaces under certain classes of sampling measure.

This paper aims at blending inverse learning theory together with recent developments in convex regularization techniques in the context of inverse problems \cite{benning2018modern, burger2004convergence}. Although there is a body of work studying methods such as Lasso and generalized approaches in learning theory (see, e.g., \cite{hastie2015statistical}), to the best our knowledge, general convex regularization has not been considered before for inverse statistical learning problems.
Here, we focus on variational regularization schemes utilizing $p$-homogeneous penalties, in particular, focusing on the case $1<p \leq 2$, and derive a framework for establishing convergence rates in expected symmetric Bregman distance. Our work is aligned with the common assumption in learning theory that the design measure, i.e., the probability distribution generating the evaluation points $(u_n)_{n=1}^N$, is unknown.

Let us consider a linear inverse problem
\begin{equation}
	\label{eq:IP}
	g = A f,
\end{equation}
where $A: X \rightarrow Y$ is a bounded linear operator between a separable Banach space $X$ and a Hilbert space $Y$. 
Furthermore, we assume that $Y$ is a function space from a subset $U \subset \R^d$ to a Hilbert space $V$.
We observe noisy point evaluations of $g$ at given points $\{u_i\}_{i=1}^N\subset U$ according to
\begin{equation}
	\label{eq:observation}
	g^\delta_i = g^\dagger(u_i) + \delta \epsilon_i
\end{equation}
for $i=1,...,N$, where $\epsilon_i$ are i.i.d.~and have suitable distribution. Moreover, the noiseless observation $g^\dagger = A f^\dagger$ corresponds to our ground truth $f^\dagger \in X$. In the following we study properties of a regularized solutions $\fdan$ defined via the variational problem
\begin{equation}
	\fdan := \argmin_{f\in X} \left\{ \frac 1{2N} \sum_{i=1}^N \norm{(Af)(u_i) - g^\delta_i}^2_{V} + \alpha R(f)\right\},
\label{eq:MinFunct}	
\end{equation}
where $R : X \to \R \cup \{\infty\}$ is a convex functional satisfying certain technical properties listed below (see assumption~\ref{ass:R}).

Our main contributions in the case of a $p$-homogeneous $R$ for $1<p\leq 2$ are as follows:
\begin{itemize}
\item In theorem \ref{thm:bregman_dist_gen} we derive an upper bound to the reconstruction error, i.e., the distance between $\fdan$ and $f^\dagger$ measured in the symmetric Bregman distance induced by $R$. This upper bound is composed of terms that generalize the approximation and sample error terms observed in the spectral regularization setting in Hilbert spaces.  
\item In addition to the standard framework usually developed for a fixed noise level $\delta$,  we discuss an interesting regime where the noise level is small compared to the number of design points,  i.e., $\delta \simeq N^{-\rho}$, $\rho>1$.  Such a setup requires modified estimates and the corresponding analysis is developed in parallel to the standard framework. We conjecture that such alternative estimates can improve the standard estimates under specific sampling schemes discussed in remark \ref{rem:martin}.
\item We derive convergence rates for a penalty $R(f) = \frac 1p \norm{f}_X^p$ under suitable assumptions in the standard framework in theorem \ref{thm:general_rate} and for the small noise regime in theorem \ref{thm:general_rate2}. 
In terms of optimality, we compare our rates with the minimax-optimal rates in the Hilbert space setting for $p=2$ obtained in \cite{blanchard2018optimal}. Restricted to Hilbert spaces and classical Tikhonov regularization scheme, our method yields convergence rates that coincide with \cite{blanchard2018optimal} only under restrictive assumptions. However, the underlying discrepancy between \cite{blanchard2018optimal} and our approach is highlighted, and we propose a modification to the method that can provide improved rates.
\item As an application of our theory, we prove a concrete convergence rate for the above case when $X$ is a Besov space $B_{pp}^s(\R^d)$ in section \ref{sec:BesovConc}. Moreover, in section \ref{subsec:embedding} we discuss how to derive convergence rates if a continuous embedding of the Banach space $X$ to some Hilbert space $X_0$ is available. If the embedding has suitable approximation properties, this approach can provide useful convergence rates.
\item We study the classical inverse problem of X-ray tomography \cite{natterer2001mathematics} under random sampling of the imaging angles and using Besov space penalties. We demonstrate that the Radon transform has suitable spectral properties making our work aligned with the optimal rates in the $p=2$ setup. Moreover, we observe the convergence rates predicted by our results in numerical simulations also for $1<p\leq 2$.
\end{itemize}

We emphasize that we do not prove minimax-optimality of our results and, in particular, no lower bounds are derived. However, we point out that similar techniques can lead to minimax-optimal rates in inverse problems when considered against suitable source conditions \cite{weidling2020optimal, burger2018large}.

Our methodology has close connections to the reproducing kernel methods, which is a popular field with a vast body of literature. Let us note that connections of kernel regression methods to regularization theory were first studied in \cite{de2006discretization, vito2005learning, gerfo2008spectral} and the line of research has since become widely popular. Early work on upper rates of convergence in a reproducing kernel Hilbert space was carried out by Smale and Cucker in \cite{cucker2002best}, where they utilized a covering number technique. After the initial success, there has been a long line of subsequent work \cite{vito2005learning, smale2005shannon, smale2007learning, bauer2007regularization, yao2007early, caponnetto2007optimal} providing convergence rates comparable to \cite{blanchard2018optimal}. 
Let us also point out that there is an avenue of research \cite{mendelson2010regularization, steinwart2009optimal} considering penalties of type $R(f) = \frac 1p \norm{f}_X^p$. Notice that the notion of convergence in the usual learning context  and the inverse problem setting is different and are not directly comparable: in learning theory the convergence rates are derived in $L^2(\mu)$ norm, where $\mu$ is the unknown sampling measure generating data points. However, since the solution and data space, i.e. $X$ and $Y$, differ for the inverse problem, it is natural to consider modes of convergence in $X$. For a related discussion and brief overview on relevant convergence rate literature, see \cite{Muecke2017}.

In terms of inverse learning problems, we mention that Tikhonov regularization of non-linear inverse problems is considered in \cite{rastogi2019convergence} and adaptive parameter choice rules are studied 
in \cite{lu2020balancing}. Moreover, for distributed learning of inverse problems, see \cite{guo2017learning} and references therein.

This paper is organized as follows. In section \ref{sec:preliminaries} we provide preliminaries of the mathematical setting and assumptions in our work. Our assumptions on the sampling setup are closely aligned with previous literature such as \cite{blanchard2018optimal}. In section
\ref{sec:bounds_on_bregman} we derive general bounds on the symmetric Bregman distance between $\fdan$ and $f^\dagger$. In section \ref{sec:p-homog} we develop these bounds further in the case $R(f) = \frac 1p \norm{f}_X^p$, which enables the use of lemma \ref{lem:xuroach} - a key ingredient of the convergence analysis. Section \ref{sec:BesovConc} discusses the Besov space regularization and derives concetration bounds. Numerical simulations of random angle X-ray tomography are presented in section \ref{sec:xray}: the Radon transform is introduced in section \ref{sub:radon} and discretized formulation is specified in section \ref{sec:discretization}. In section \ref{sec:numerical_exp} we provide the numerical experiments related to the convergence rates.

\section{Preliminaries}
\label{sec:preliminaries}

Let us briefly define some notation.
For two functions $f,g: X\to \R$, we write $f\lesssim g$ if there exists a universal constant $C>0$ such that $f(x) \leq C g(x)$ for all $x\in X$.  Similarly,  we write $f\simeq g$ if it holds that $g\lesssim f \lesssim g$.
Below, $C>0$ will denote a generic constant unless otherwise specified in the context.

\subsection{Sampling operator}

Consider $Y = L^2(U;V)$, where $U \subset \R^d$ and $V$ is a Hilbert space. Here, we assume that the range of the operator $A$ is contained in a Banach space $Z \subset Y$ such that $Z \subset \mathcal{C}(U;V)$ continuously, and that $A : X \to Z$ is continuous. Notice that $\mathcal{C}(U;V)$ is a Banach space and  $Z \subset Y \subset Z^*$ forms a Gelfand triplet. In particular, we have that
\begin{equation}
	\langle g_1, g_2\rangle_{Z^*\times Z} = \langle g_1,g_2 \rangle_Y
	\label{gelfand}
\end{equation}
for any $g_1 \in Y$ and $g_2\in Z$. The main motivation for the assumption above is that point evaluations of $g^\dagger$ in equation \eqref{eq:observation} are well-defined.

Following previous work \cite{caponnetto2007optimal, blanchard2018optimal}, for any $u\in U$ we define
a sampling operator $A_u \in \LO(Z, V)$ such that
\begin{equation*}
	A_u f = (Af)(u).
\end{equation*}
We make the following assumption:

\begin{assumption}
\label{ass:sampling}
\begin{itemize}
\item[(a)] There exists $\kappa\leq 1$ such that for all $u\in U$ and for all $f \in X$ we have
\begin{equation*}
	\norm{A_u f}_V \leq \kappa \norm{f}_X.
\end{equation*}
\item[(b)] The mapping
\begin{equation*}
	u \mapsto (Af)(u)
\end{equation*}
is measurable for all $f\in X$.
\end{itemize}
\end{assumption}
Notice that the assumption of $\kappa\leq 1$ is not restrictive; for larger values of $\kappa$ one can renormalize the problem in the spirit of \cite{caponnetto2007optimal, blanchard2018optimal}.

Next we consider sampling at multiple design points $\{u_i\}_{i=1}^N \subset U$. Let us introduce the following notation for the product spaces $U_N = \oplus_{i=1}^N U$ with the usual topology and $V_N = \oplus_{i=1}^N V$ with the inner product 
\begin{equation*}
	\langle {\bf v}, \tilde {\bf v}\rangle_{V_N} = \frac 1N \sum_{j=1}^N \langle v_j, \tilde v_j\rangle_V,
\end{equation*}
where ${\bf v} = (v_j)_{j=1}^N, \tilde {\bf v} = (\tilde v_j)_{j=1}^N \in V_N$.
The multiple sampling operator $\Abu \in \LO(Z, V_N)$ is defined by
\[
\Abu f = (A_{u_i} f)_{i=1}^N \in V_N
\]
for ${\bf u} = \{u_i\}_{i=1}^N$.
Notice carefully that
\begin{equation*}
	\Abu^* {\bf v} = \frac 1{N} \sum_{i=1}^N A_{u_i}^* v_i \quad \text{for any} \; {\bf v} = (v_i)_{i=1}^N \in V_N.
\end{equation*}
Moreover, in the following it is convenient to introduce the following notation for the \emph{normal sampling operator} 
$\Bbu := \Abu^* \Abu \in  \LO(X, X^*)$.

In the following, we will consider the design points $\{u_i\}_{i=1}^N$ as a random sample drawn from a probability distribution $\mu$. Thus, let $\mu$ be a probability measure on $U$ and define the corresponding weighted space $Y_\mu = L^2(U, \mu; V)$ as a Hilbert space induced by the inner product
\begin{equation*}
	\langle g_1, g_2 \rangle_{Y_\mu} := \int_U \langle g_1(u), g_2(u) \rangle_V \mu(du).
\end{equation*}
Clearly, $Z \subset {\mathcal C}(U; V) \subset Y_\mu$ and we denote
\begin{equation*}
	A_\mu = \iota A : X \to Y_\mu,
\end{equation*}
where $\iota : Z \to Y_\mu$ is the canonical injection map. As a simple example, one can consider a uniform distribution on a bounded domain $U$.
Obviously, in such a case, the inner products of $Y_\mu$ and $Y$ coincide up to a constant.  

\subsection{Problem setting}

For the ground truth $f^\dagger \in X$ we define a noise-free observations $\gn$ by
\begin{equation}
	\label{eq:noisefreedata}
	\gn := \Abu f^\dagger,
\end{equation}
and the noisy observation $\gdan$ as
\begin{equation}
\label{eq:noisydata}
	\gdan := \Abu f^\dagger + \delta\epsilon_N,
\end{equation}
where $\delta>0$ is the noise level, $\epsilon_N = (\epsilon_N^i)_{i=1}^N \in V_N$ is a random variable such that $\epsilon_N^i\sim \epsilon$ i.i.d.~where $\epsilon$ is independent of $\mu$, zero-mean and satisfies
\begin{equation}
\label{eq:randomnoise}
\E \norm{\epsilon}_{V}^m < \frac 12 m! M^{m-2}
\end{equation}
for all $m\geq 2$ and some constant $M>0$. Notice that here $\delta>0$ plays the role of standard deviation that is usually included in the Bernstein-type assumptions \eqref{eq:randomnoise} on the observational noise. To build intuition, we point out that a normally distributed $\epsilon$ in $V=\R$ satisfies \eqref{eq:randomnoise} with $M=1$.

In the following we consider regularized solutions $\fdan$ to problems \eqref{eq:noisefreedata} and \eqref{eq:noisydata}
given by
\begin{equation}
	\label{eq:regularized_sol_R}
	\fdan \in  \argmin_{f\in X} \Jdan(f) := \argmin_{f\in X} \left\{ \frac 12 \norm{\Abu f - \gdan}^2_{V_N} + \alpha R(f)\right\}.
\end{equation}
A regularized solution for the noise-free data is denoted by $f_{\alpha,N}$. Notice that we do not require the minimizers of such problems to be unique at this stage.
\begin{assumption}
\label{ass:R}
The convex functional $R : X \to \R \cup \{ \infty\}$ satisfies the following four condition:
\begin{itemize}
\item[(R1)] the functional $R$ is lower semicontinuous in some topology $\tau$ on $X$;
\item[(R2)] the sublevel sets $M_r = \{R \leq r\}$ are sequentially compact in the topology $\tau$ on $X$;
\item[(R3)] the convex conjugate $R^\star$ is finite on a ball in $X^*$ centered at zero;
\item[(R4)] $R(-f) = R(f)$ for all $f\in X$.
\end{itemize}
\end{assumption}
The results contained in the next section are set in a deterministic framework. However, in the following we consider more specific functionals $R$ that ensure uniqueness of $\fdan$ and comment on the measurability of the learning method.
Notice that the symmetry condition (R4) is not necessary, but is employed to make the results more accessible.

\section{Bounds on the Bregman distance}
\label{sec:bounds_on_bregman}

The optimality criterion associated with \eqref{eq:regularized_sol_R} is given by
\begin{equation}
	\label{eq:optimality_criterion}
	\Abu^* ( \Abu \fdan - \gdan) + \alpha \pdan = 0
\end{equation}
for $\pdan \in \partial R(\fdan)$, where $\partial R$ denotes the subdifferential:
\begin{equation*}
	\partial R(f) = \{\sdiff \in X^* \; | \; R(f)-R(\tilde f) \leq \langle \sdiff, f-\tilde f\rangle_{X^* \times X} \; \text{for all} \; \tilde f \in X\}.
\end{equation*}
Moreover, for $\sdiff_f\in \partial R(f)$ and $\sdiff_{\tilde f} \in \partial R(\tilde f)$ we define the symmetric Bregman distance between $f$ and $\tilde f$ as
\begin{equation*}
	D^{\sdiff_f, \sdiff_{\tilde f}}_R (f,\tilde f) = \langle \sdiff_f-\sdiff_{\tilde f}, f-\tilde f\rangle_{X^*\times X}.
\end{equation*}
When the subdifferential elements $\sdiff_f\in \partial R(f)$ are unique, we will drop the dependence on the subgradients in the notation of the symmetric Bregman distance and write simply $D_R(f,\tilde f)$. 

\begin{prop}[A-priori estimates]
\label{prop:apriori}
Let $R$ satisfy the Assumption \ref{ass:R}. Then the functional $\Jdan$ has a minimizer. Any minimizer $\fdan \in X$ of $\Jdan$ satisfies
\begin{equation}
\label{eq:apriori}
	R(\fdan) \leq  R(f^\dagger) + \frac{\delta^2}{2\alpha} \norm{\epsilon_N}^2_{V_N}.
\end{equation}
In addition, if $R$ is $p$-homogeneous with $p>1$ we have for some constant $C>0$ that
\begin{equation}
\label{eq:apriori2}
	R(\fdan) \leq C\left(R(f^\dagger) + \left(\frac{\delta}\alpha\right)^{\frac{p}{p-1}} R^\star( \Abu^* \epsilon_N)\right).
\end{equation}
\end{prop}

\begin{proof}
Consider the sublevel set $M = \{f \in X\; | \; \Jdan(f) \leq \Jdan(f^\dagger) \}$.
Now, any $f \in M$ satisfies
\begin{eqnarray}
\label{eq:aprioriaux1}
\frac 12 \norm{\Abu(f-f^\dagger)}^2_{V_N} + \alpha R(f)
& = & \Jdan(f) + \langle \Abu (f - f^\dagger), \delta \epsilon_N\rangle_{V_N} - \frac{\delta^2}{2} \norm{\epsilon_N}^2_{V_N} \nonumber \\
& \leq & \Jdan(f^\dagger) + \langle \Abu (f-f^\dagger), \delta \epsilon_N\rangle_{V_N} - \frac{\delta^2}2 \norm{\epsilon_N}^2_{V_N} \nonumber \\
& = & \alpha R(f^\dagger) + \langle \Abu (f-f^\dagger), \delta \epsilon_N\rangle_{V_N} \nonumber \\
& \leq & \alpha R(f^\dagger)  + \frac 12 \norm{\Abu(f-f^\dagger)}^2_{V_N} + \frac{\delta^2}{2} \norm{\epsilon_N}_{V_N}^2,
\end{eqnarray}
which yields the estimate \eqref{eq:apriori}. The second estimate follows by applying the generalized Fenchel--Young's inequality
in the second to last expression in \eqref{eq:aprioriaux1}, namely,
\begin{eqnarray}
	\label{eq:aprioriaux2}
	 \langle \Abu (f-f^\dagger), \delta \epsilon_N\rangle_{V_N} & = &  \langle f-f^\dagger, \delta \Abu^* \epsilon_N\rangle_{V_N} \nonumber \\
	 & \leq & R\left(\beta \alpha^{\frac 1p}(f-f^\dagger)\right) + R^\star\left(\frac \delta{\beta\alpha^{\frac 1p}} \Abu^* \epsilon_N\right) \nonumber \\
	 & \leq & C \beta^p \alpha\left(R(f) + R(f^\dagger)\right) + \frac{\delta^{\frac p{p-1}}}{\beta^{\frac p{p-1}} \alpha^{\frac 1{p-1}}} R^\star(\Abu^* \epsilon_N),
\end{eqnarray}
where $\beta>0$ is an arbitrary constant and the triangle inequality for $R$ follows due to convexity and homogeneity with some constant $C>0$ depending on $p$. Applying inequality \eqref{eq:aprioriaux2} to the second to last expression in \eqref{eq:aprioriaux1} and setting $\beta^p = \frac 1{2C}$ yields the a priori estimate \eqref{eq:apriori2} after a division by $\alpha$.

The existence of the minimizer follows by standard arguments. Assume that $\{f_j\}_{j=1}^\infty \subset M$ is a minimizing sequence of $\Jdan$. By the assumption (R2) we can extract a converging subsequence $f_{j_k} \to \tilde f \in X$. Finally, $\tilde f$ is a minimizer due to the assumption (R1).
\end{proof}

\begin{prop}
\label{prop:aux_convex2}
Suppose assumption \ref{ass:R} is satisfied.
Then any regularized solution $\fdan$ given by \eqref{eq:regularized_sol_R} satisfies
\begin{multline}
	\label{eq:aux_breg2}
	D^{\pdan, \sdiff^\dagger}_R(\fdan, f^\dagger) \
	\leq  \inf_{\bar w \in V_N} \left(R^\star\left((\Gamma_1^{-1})^*(\sdiff^\dagger - \Abu^* \bar w)\right) + \frac \alpha 2 \norm{\bar w}_{V_N}^2\right) + R(\Gamma_1(f^\dagger - \fdan)) \\
	 + \frac{1}{\alpha} \left(R^\star\left(\delta (\Gamma_2^{-1})^* \Abu^* \epsilon_N\right)  +  R\left(\Gamma_2(f^\dagger - \fdan)\right)\right),
\end{multline}
where $\pdan \in \partial R(\fdan)$, $\sdiff^\dagger \in \partial R(f^\dagger)$ and $\Gamma_1, \Gamma_2 : X\to X$ are arbitrary linear invertible operators.
\end{prop}

\begin{proof}
Let us apply the data-generating distribution of $\gdan$ given in \eqref{eq:noisydata} to the optimality criterion
\eqref{eq:optimality_criterion} and substract $\sdiff^\dagger$ on both sides to obtain
\begin{equation*}
	\Bbu (\fdan - f^\dagger) + \alpha (\pdan-\sdiff^\dagger) = - \alpha \sdiff^\dagger + \delta \Abu^* \epsilon_N.
\end{equation*}
Now taking dual pairing with $\fdan-f^\dagger$ on both sides yields
\begin{multline}
	\label{eq:aux_convex1}
	\norm{\Abu(\fdan - f^\dagger)}_{V_N}^2 + \alpha D^{\pdan, \sdiff^\dagger}_R(\fdan, f^\dagger) \\
	= \alpha \langle \sdiff^\dagger, f^\dagger - \fdan\rangle_{X^*\times X} + \delta \langle \Abu^* \epsilon_N, \fdan - f^\dagger  \rangle_{X^* \times X}
\end{multline}
Applying both standard Young's and Fenchel--Young's inequalities to the first term on the right hand side now yields, for any $\bar w \in V_N$
\[
\begin{aligned}
\alpha & \langle \sdiff^\dagger, f^\dagger - \fdan\rangle_{X^*\times X} \\
& = 	\alpha \langle (\Gamma_1^{-1})^*(\sdiff^\dagger -\Abu^* \bar w), \Gamma_1(f^\dagger - \fdan)\rangle_{X^*\times X} +  \alpha \langle \bar w, \Abu( f^\dagger - \fdan)\rangle_{V_N} \\
& \leq \alpha  R^\star\left((\Gamma_1^{-1})^*(\sdiff^\dagger - \Abu^* \bar w) \right) + \alpha  R(\Gamma_1(f^\dagger - \fdan)) +  \frac{\alpha^2} 2 \norm{\bar w}_{V_N}^2
+ \frac 12 \norm{\Abu(\fdan - f^\dagger)}_{V_N}^2,
\end{aligned}
\]
where we introduced an arbitrary invertible linear operator $\Gamma_1: X\to X$.
Similarly, the second term on the right hand side of \eqref{eq:aux_convex1} can be bounded by
\begin{eqnarray}
	\label{eq:aux_propbound2}
	\delta \langle \Abu^* \epsilon_N, f^\dagger - \fdan\rangle_{X^* \times X} 
	& = & \delta \langle (\Gamma_2^{-1})^* \Abu^* \epsilon_N, \Gamma_2(\fdan - f^\dagger)\rangle_{X^* \times X} \nonumber \\
	& \leq & R^\star((\delta \Gamma_2^{-1})^*\Abu^* \epsilon_N) + R(\Gamma_2(f^\dagger - \fdan))
\end{eqnarray}
Now dividing by $\alpha$ on both sides yields the claim.
\end{proof}

An alternative bound for the Bregman distance between $\fdan$ and $f^\dagger$ is provided in the following result.

\begin{prop}
\label{prop:aux_convex}
Suppose assumption \ref{ass:R} is satisfied.
Then any regularized solution $\fdan$ given by \eqref{eq:regularized_sol_R} satisfies
\begin{multline}
	\label{eq:aux_breg1}
	D^{\pdan, \sdiff^\dagger}_R(\fdan, f^\dagger) \\
	\leq  \inf_{\bar w \in V_N} \left(R^\star\left((\Gamma^{-1})^*(\sdiff^\dagger - \Abu^* \bar w)\right) + \frac \alpha 2 \norm{\bar w}_{V_N}^2\right) + R(\Gamma(f^\dagger - \fdan)) + \frac{\delta^2}{2 \alpha} \norm{\epsilon_N}_{V_N}^2,
\end{multline}
where $\pdan \in \partial R(\fdan)$, $\sdiff^\dagger \in \partial R(f^\dagger)$ and $\Gamma: X\to X$ is an arbitrary invertible linear operator.
\end{prop}

\begin{proof}
The proof is identical to the previous proposition. However, we apply the bound
\begin{equation}
	\label{eq:aux_propbound1}
	\delta \langle \Abu^* \epsilon_N, f^\dagger - \fdan\rangle_{X^* \times X} \leq \frac {\delta^2}2 \norm{\epsilon_N}_{V_N}^2 + \frac 12 \norm{\Abu(\fdan - f^\dagger)}_{V_N}^2.
\end{equation}
instead of \eqref{eq:aux_propbound2}.
\end{proof}

\begin{rem}
Let us note that in what follows the propositions \ref{prop:aux_convex2} and \ref{prop:aux_convex} yield different convergence rates. At this stage, this can be seen by observing the variance terms appearing in the above propositions, which are $\frac 1\alpha R^\star (\delta (\Gamma_2^{-1})^* \Abu^* \epsilon_N)$ and $\frac{\delta^2}{2\alpha}\norm{\epsilon_N}_{V_N}^2$, respectively. The expectation of the former decays w.r.t.~$N$, whereas the expectation of the latter is independent of $N$. Therefore, for a fixed noise level $\delta$ one cannot expect convergence with the bounds developed based on proposition \ref{prop:aux_convex}. On the other hand, the rest of the terms of the upper bound (generalized approximation error) are rather similar in both propositions, which will imply different balancing properties for the estimate and therefore different convergence regimes.
\end{rem}

\section{The $p$-homogeneous regularizer and convergence rates}
\label{sec:p-homog}

From here on, we consider a $p$-homogenous regularizer
\begin{equation}
	\label{eq:R=phomog}
	R(f) = \frac 1p \norm{f}_X^p
\end{equation}
with $1 < p< \infty$. In this case, the subdifferential sets consist of unique single points and therefore, in the following, we drop the related notation from the Bregman distance. Moreover, due to the strict convexity of the functional $R$, the regularized solution $\fdan$ of \eqref{eq:regularized_sol_R} is unique. In addition, the mapping 
\begin{equation*}
	(f, (u_i, g_i^\delta)_{i=1}^N) \mapsto \frac 12 \norm{\Abu f - \gdan}^2_{V_N} + \alpha R(f)
\end{equation*}
is continuous and the measurability of $(\bu, \epsilon_N) \mapsto \fdan$ with respect to the universal completion of the product $\sigma$-algebra of $(U\times V)^N$ follows by Aumann's measurable selection principle (analogous to \cite[Lemma 6.23]{steinwart2008support}). In particular, following the definition \cite[Definition 6.2]{steinwart2008support} the introduced learning method is measurable. 

\subsection{General bounds}
We start by introducing the following notations. Let us abbreviate
\begin{equation*}
	E_{\beta, \bu}(\bar w; \sdiff^\dagger) := R^\star\left(\sdiff^\dagger - \Abu^* \bar w\right) + \frac \beta 2 \norm{\bar w}_{V_N}^2.
\end{equation*}
and, denoting by $\sdiff^\dagger$ the only element of $\partial R(f^\dagger)$, set
\begin{equation*}
	\calR(\beta, \bu; f^\dagger) = \inf_{\bar w\in V_N} E_{\beta, \bu}(\bar w; \sdiff^\dagger).
\end{equation*}
As we will notice, the term $\calR$ is related to the source condition, which we will discuss later.
Next, we develop propositions \ref{prop:aux_convex2} and \ref{prop:aux_convex} further by combining them with a priori bounds and the following technical lemma that applies to $p$-homogeneous functionals.

\begin{lemma}
\label{lem:xuroach}
Suppose $R$ is of the form \eqref{eq:R=phomog} and let $f,\tilde f\in X$.
It follows that for $p=2$ we have
\begin{equation*}
	R(f-\tilde f) = \frac 12 D_R(f,\tilde f)
\end{equation*}
and for $1<p<2$ it holds that
\begin{equation*}
	\gamma^p R(f-\tilde f) \leq C \left(1-\frac p2\right) \gamma^{\frac{2p}{2-p}} \max\left\{R(f), R(\tilde f)\right\} + \frac p2 D_R(f,\tilde f),
\end{equation*}
for some $C>0$ depending on $p$ with any $\gamma>0$.
\end{lemma}

\begin{proof}
The case $p=2$ is trivial. For $1<p<2$ consider
the Xu--Roach inequality II \cite[Thm. 2.40(b)]{schuster2012regularization} in $X$ that yields
\begin{equation*}
	D_R(f,\tilde f) \geq C\max\left\{\norm{f}_X, \norm{\tilde f}_X\right\}^{p-2} \norm{f-\tilde f}_X^2 = Cp  \max\left\{\norm{f}_X, \norm{\tilde f}_X\right\}^{p-2} R(f-\tilde f)^{\frac 2p},
\end{equation*}
and, therefore,
\begin{equation*}
	R(f-\tilde f) \leq \frac Cp \gamma^{\frac{p^2}2} \max\left\{\norm{f}_X, \norm{\tilde f}_X\right\}^{\frac{p(2-p)}{2}} \cdot \gamma^{-\frac{p^2}2}  D_R(f,\tilde f)^{\frac p2}
\end{equation*}
for any $\gamma >0$.
Next, applying Young's inequality to the right-hand side with H\"{o}lder conjugates $\left(\frac{2}{2-p},\frac 2p\right)$ yields the claim.
\end{proof}

Now we are ready to establish more concrete bounds based on propositions \ref{prop:aux_convex2} and \ref{prop:aux_convex}. Let us first consider the quadratic case $p=2$.

\begin{thm}
\label{thm:bregman_dist_quad}
Suppose that assumption \ref{ass:R} holds and $R$ is of the form \eqref{eq:R=phomog} with $p=2$.
Then the regularized solution $\fdan$ given by \eqref{eq:regularized_sol_R} satisfies
\begin{equation}
\label{eq:bregman_dist_quad}
	D_R(\fdan, f^\dagger) \leq  \min\left\{
	4 \calR\left(\frac \alpha 2, N; f^\dagger\right) + \frac{2 \delta^2}{\alpha^2} \norm{\Abu^* \epsilon_N}_{X^*}^2,
	2 \calR\left(\alpha, N; f^\dagger\right) + \frac{\delta^2}{\alpha} \norm{\epsilon_N}_{V_N}^2
	\right\}.
\end{equation}
\end{thm}

\begin{proof}
In the statement of proposition \ref{prop:aux_convex2}, consider $\Gamma_1 = \gamma_1 I$ and $\Gamma_2 = \gamma_2 I$. By the homogeneity of $R$, and applying the first estimate in lemma \ref{lem:xuroach}, we have that
\begin{equation*}
	\left(1-\frac{\gamma_1^2}2 - \frac{\gamma_2^2}{2\alpha}\right) D_R(\fdan, f^\dagger) \leq\gamma_1^{-2} \calR(\alpha \gamma_1^2, N; \sdiff)
	+ \frac{\delta^2}{2 \alpha \gamma_2^2} \norm{\Abu^* \epsilon_N}_{X^*}^2
\end{equation*}
Now setting $\gamma_1^2 = \frac{\gamma_2^2}{\alpha} = \frac 12$ yields
\begin{equation*}
	\frac 12 D_R(\fdan, f^\dagger) \leq 2 \calR\left(\frac \alpha 2, N; f^\dagger\right) + \frac{\delta^2}{\alpha^2} \norm{\Abu^* \epsilon_N}_{X^*}^2.
\end{equation*}
Similarly, starting from the statement of proposition \ref{prop:aux_convex}, setting $\Gamma = \gamma I$, via lemma \ref{lem:xuroach} we have
\begin{equation*}
	\left(1- \frac{\gamma^2}2\right) D_R(\fdan, f^\dagger) \leq \gamma^{-2} \calR(\alpha \gamma^2, N; f^\dagger) + \frac{\delta^2}{2\alpha} \norm{\epsilon_N}_{V_N}^2
\end{equation*}
Setting $\gamma^2 = 1$ yields the second part of the claim. 
This completes the proof.
\end{proof}

\begin{thm}
\label{thm:bregman_dist_gen}
Suppose that assumption \ref{ass:R} holds and $R$ is of the form \eqref{eq:R=phomog} with $1<p<2$.
Then the regularized solution $\fdan$ given by \eqref{eq:regularized_sol_R} satisfies
the following two inequalities:
\begin{itemize}
\item[(i)] It holds that
\begin{multline}
	\label{eq:bregman_dist_gen2}
	D_R(\fdan, f^\dagger) \\
	\leq   \widetilde C_p\left[\gamma_1^{-q} \calR(\alpha \gamma_1^q, \bu; f^\dagger)  + 
	 H(\alpha, \delta, \gamma_1, \gamma_2) R^\star(\Abu^* \epsilon_N) +
	 \left(\gamma_1^p	+ \frac{\gamma_2^p}{\alpha}\right)^{\frac{2}{2-p}} R(f^\dagger)\right]
\end{multline}
for arbitrary $\gamma_1,\gamma_2 >0$, where $\widetilde C_p>0$ is a constant dependent on $p$, 
\begin{equation}
	\label{eq:paramH}
	H(\alpha, \delta, \gamma_1, \gamma_2) = \frac{\delta^q}{\alpha \gamma_2^q} + \left(\gamma_1^p	+ \frac{\gamma_2^p}{\alpha}\right)^{\frac{2}{2-p}} \left(\frac \delta \alpha\right)^{q}.
\end{equation}
\item[(ii)] We have
\begin{equation}
	\label{eq:bregman_dist_gen1}
	D_R(\fdan, f^\dagger) \leq C_p\left(\gamma^{-q} \calR(\alpha \gamma^q, \bu; f^\dagger)  + \frac{\delta^2}{\alpha}\left(1 + \gamma^{\frac{2p}{2-p}}\right) \norm{\epsilon_N}_{V_N}^2 + \gamma^{\frac{2p}{2-p}} R(f^\dagger)\right)
\end{equation}
for arbitrary $\gamma>0$, where $C_p>0$ is a constant dependent on $p$ 
and $(p,q)$ are H\"older conjugates.
\end{itemize}
\end{thm}

\begin{proof}
Consider the first claim.
Applying lemma \ref{lem:xuroach} and the second a priori bound in proposition \ref{prop:apriori} to proposition \ref{prop:aux_convex2} we have that
\begin{multline*}
\left(1-\frac p2\right) D_R(\fdan, f^\dagger) \\ \leq \gamma_1^{-q} \calR(\alpha \gamma_1^q, \bu; f^\dagger) + \frac{\delta^q}{\alpha \gamma_2^q} R^\star(\Abu^*\epsilon_N) + C\left(1-\frac p2\right)\left(\gamma_1^p + \frac{\gamma_2^p}{\alpha}\right)^{\frac{2}{2-p}} \max\{R(\fdan), R(f^\dagger)\} \\
\leq \gamma_1^{-q} \calR(\alpha \gamma_1^q, \bu; f^\dagger)  + 
\left[ \frac{\delta^q}{\alpha \gamma_2^q} + C\left(1-\frac p2\right) \left(\gamma_1^p + \frac{\gamma_2^p}{\alpha}\right)^{\frac{2}{2-p}} \left(\frac\delta\alpha\right)^q\right] R^\star(\Abu^*\epsilon_N) \\
+ C\left(1-\frac p2\right)  \left(\gamma_1^p + \frac{\gamma_2^p}{\alpha}\right)^{\frac{2}{2-p}} R(f^\dagger).
\end{multline*}

For the second inequality we deduce similarly
applying lemma \ref{lem:xuroach} and the first a priori bound in proposition \ref{prop:apriori} to proposition \ref{prop:aux_convex} that
\begin{multline*}
\left(1-\frac p2\right) D_R(\fdan, f^\dagger) \\ \leq \gamma^{-q} \calR(\alpha \gamma^q, \bu; f^\dagger) + \frac{\delta^2}{2 \alpha}\norm{\epsilon_N}_{V_N}^2 + C\left(1-\frac p2\right) \gamma^{\frac{2p}{2-p}} \max\{R(\fdan), R(f^\dagger)\} \\
\leq \gamma^{-q} \calR(\alpha \gamma^q, \bu; f^\dagger)  + \frac{\delta^2}{2 \alpha}\left(1 + C\left(1-\frac p2\right) \gamma^{\frac{2p}{2-p}}\right) \norm{\epsilon_N}_{V_N}^2 + C\left(1-\frac p2\right)\gamma^{\frac{2p}{2-p}} R(f^\dagger),
\end{multline*}
which yields inequality \eqref{eq:bregman_dist_gen1} after dividing by $1-\frac p2$.
This completes the proof.
\end{proof}

\subsection{The case $p=2$ in Hilbert spaces}
\label{subsec:Tikhonov}

In this section we compare our technique developed above and the convergence rates it implies to the optimal convergence rates known in 
the case when $X$ is a Hilbert space and $p=2$. Although optimal rates are known for general spectral regularization schemes \cite{blanchard2018optimal}, our key message below can be demonstrated by considering classical Tikhonov regularization
\begin{equation}
	\label{eq:R=quad}
	R(f) = \frac 12 \norm{f}_X^2.
\end{equation}

We note that this setting implies that if $r \in \partial R(f)$, then $r=f$. Moreover, recall that $X^* = X$.

\begin{lemma}
\label{lem:Rforquad}
Suppose $X$ is an Hilbert space and $R$ satisfies \eqref{eq:R=quad}. Then we have
\begin{equation*}
	\calR(\beta, \bu; f^\dagger) = \frac \beta 2 \norm{(\Bbu + \beta I)^{-\frac 12} f^\dagger}_X^2.
\end{equation*}
\end{lemma}
\begin{proof}
Recall that
\begin{equation*}
	\calR(\beta, \bu; f^\dagger) = \inf_{\bar w \in V_N} \left(\frac 12 \norm{f^\dagger - \Abu^* \bar w}_X^2 + \frac \beta 2\norm{\bar w}_{V_N}^2\right).
\end{equation*}
The minimizing element on the right hand side naturally satisfies
\begin{equation*}
	\bar w_{\text{inf}} = (\Abu \Abu^* + \beta I)^{-1} \Abu f^\dagger
\end{equation*}
Next, we have that
\begin{eqnarray*}
	2E_{\beta, N}(\bar w_{\text{inf}}, f^\dagger) & = & \norm{f^\dagger - \Abu^* \bar w_{\text{inf}}}_X^2 + \beta \norm{\bar w_{\text{inf}}}_{V_N}^2 \\
	& = & \norm{ f^\dagger}_X^2 - 2\langle f^\dagger, \Abu^* \bar w_{\text{inf}} \rangle 
	+ \langle (\Abu  \Abu^* + \beta I)  \bar w_{\text{inf}},  \bar w_{\text{inf}}\rangle \\
	& = & \norm{ f^\dagger}_X^2 - \langle  f^\dagger,  \Abu^* \bar w_{\text{inf}} \rangle \\
	& = & \langle  f^\dagger,  \left(I - \Abu^* (\Abu  \Abu^* + \beta I)^{-1} \Abu \right) f^\dagger\rangle \\
	& = & \langle  f^\dagger,  \left(I - (\beta I + \Bbu )^{-1}  \Bbu \right)    f^\dagger\rangle \\
	& = & \beta \langle f^\dagger, (\beta I + \Bbu )^{-1} f^\dagger\rangle.
\end{eqnarray*}
This yields the result.
\end{proof}

The previous lemma enables us to prove sharp concentration results for the $\calR$ term based on techniques developed in previous inverse learning theory literature. Towards this end, let us introduce so-called \emph{effective dimension} which is defined by
\begin{equation}
	{\mathcal N}(\alpha) = \tr\left[(B_\mu + \alpha)^{-1} B_\mu\right]
	\label{eq:effDim}
\end{equation}
in the Hilbert space setting. Moreover, the source condition is typically characterized by restricting the ground truth to the subset
\begin{equation*}
	\widehat{\Omega}(s,L) = \{f\in X \; | \; f = B_\mu^s w,\;  \norm{w}_X \leq L\} \subset X.
\end{equation*}
Later on, we will focus to a more specific set
\begin{equation}
	\widetilde \Omega(L) = \{f \in X \; | \; f = A_\mu^* \tilde w,\;  \norm{\tilde w}_{Y_\mu} \leq L, \; \tilde w\in Z\} \subset X.
	\label{eq:SC_Tikhonov}
\end{equation}
Now we are ready to prove the following result.

\begin{prop}
\label{prop:hilbert_term1}
Let us define
\begin{equation}
	\label{eq:Bnbeta}
	{\mathcal B}_N(\beta) := 1 + \left( \frac{2}{N\beta} + \sqrt{\frac{{\mathcal N}(\beta)}{N\beta}}\right)^2
\end{equation}
for any $\beta>0$ and $N\in \N$. We assume that $f^\dagger \in \widehat{\Omega}(s,L)$ for $s\leq \frac 12$. It follows that
\begin{equation}
	\label{eq:calR_est1}
	\E \calR(\beta, \bu; f^\dagger) \leq  C L^2 \beta^{2s} {\mathcal B}_N(\beta)^{2s}
\end{equation}
for some constant $C>0$ independent of $\beta, N$ and $L$. Also, if $f^\dagger \in \widetilde \Omega(L)$, it holds that
\begin{equation}
	\label{eq:calR_est2}
	\E \calR(\beta, \bu; f^\dagger) \leq C L^2\left(\beta  + \frac 1 N\right)
\end{equation}
for some $C>0$ with any $\beta>0$.
\end{prop}

\begin{proof}
By lemma \ref{lem:Rforquad} we have that, with probability larger than $1-\eta$,
\begin{eqnarray*}
	\calR(\beta, \bu; f^\dagger)  & =  & \frac \beta 2 \norm{(\beta I + \Bbu)^{-\frac 12} B_\mu^s w}_X^2 \\
	& \leq & \frac \beta 2 \norm{(\beta I + \Bbu)^{s-\frac 12}}^2 \norm{(\beta I + \Bbu)^{-s} B_\mu^s}^2 L^2 \\
	& \leq & \frac \beta 2 \norm{(\beta I + \Bbu)^{s-\frac 12}}^2 \norm{(\beta I + \Bbu)^{-1} B_\mu}^{2s} L^2 \\
	& \leq & C' \beta L^2 \norm{(\beta I + \Bbu)^{s-\frac 12}}^2   \norm{(\beta I + \Bbu)^{-1}(\beta I + B_\mu)}^{2s}  \\
	& \leq & C' \beta L^2 \norm{(I + \beta^{-1}\Bbu)^{s-\frac{1}{2}}}^{2} \beta^{2s-1}  {\mathcal B}_N(\beta)^{2s} \log^{4s}\left(\frac 2\eta\right)  \\
	& \leq & C' L^2 \beta^{2s}  {\mathcal B}_N(\beta)^{2s} \log^{4s}\left(\frac 2\eta\right) ,
\end{eqnarray*}
where $C'$ is a constant and we applied propositions \ref{prop:cordes} and \ref{prop:app_prob_op_bound}. Now the claim follows by lemma \ref{lem:app_expec}.

For the purpose of the second claim, let us introduce the point evaluation operator $\Sbu : Z \to V_N$ such that
\begin{equation}
	\Sbu f = (f(u_n))_{n=1}^N \in V_N.
	\label{eq:operatorS}
\end{equation}
Clearly, we have the operator identity $\Abu = \Sbu A$. 
The term $\calR$ can be bounded by setting $\bar w = \Sbu \tilde w$, which yields
\begin{equation*}
	\calR(\beta, \bu; f^\dagger) \leq \norm{(A_\mu^* - \Abu^* \Sbu) \tilde w}_X^2 + \frac \beta 2 \norm{\Sbu \tilde w}_{V_N}^2.
\end{equation*}
We first observe that
\begin{equation}
	\label{eq:auxaux_hilbert}
	\E \norm{\Sbu w}_{V_N}^2 = \frac 1N \sum_{n=1}^N \E \norm{w(u_n)}_V^2 = \frac 1N \sum_{n=1}^N \norm{w}_{Y_\mu}^2 =  \norm{w}_{Y_\mu}^2.
\end{equation}
Moreover, since the design points are independent, we have
\begin{equation*}
	\E \norm{(A_\mu^* - \Abu^* \Sbu)\tilde w}_X^2 = \frac 1N \E \norm{(A_\mu^* - A_{u}^* S_{u})\tilde w}_X^2 \leq \frac CN \norm{\tilde w}_{Y_\mu}^2,
\end{equation*}
where the last inequality follows due to assumption \ref{ass:sampling} and the argument in \eqref{eq:auxaux_hilbert}.
\end{proof}

\begin{prop}
\label{prop:hilbert_term2}
Let $N\in \N$. For any $0<\eta<1$ there exists $C=C(M)>0$ such that
\begin{equation*}
	\E \norm{\Abu^* \epsilon_N}^2_X \leq \frac C N.
\end{equation*}
\end{prop}

\begin{proof}
We write 
\begin{equation*}
	\xi(u, \epsilon) := A_u^* \epsilon \in V
\end{equation*}
for $u\in U$ and $\epsilon \in V$. We notice that $\E \xi = 0$ and 
\begin{equation*}
	\Abu^* \epsilon_N = \frac 1N \sum_{n=1}^N \xi(u_n,\epsilon_n).
\end{equation*}
From boundedness of $A_u$ and our condition on the observational noise \eqref{eq:randomnoise}, it follows now that 
\begin{equation*}
	\E \norm{\xi}_X^m \leq \E \norm{\epsilon}_V^m \leq \frac 12 m! M^{m-2}.
\end{equation*}
Now by proposition \ref{prop:concentration_res} we have
\begin{equation*}
	\Prob\left(\norm{\Abu^* \epsilon_N}_X \geq \frac{C}{\sqrt{N}}\log \frac 2\eta \right) \leq \eta
\end{equation*}
for any $\eta \in (0,1]$.
In consequence, lemma \ref{lem:app_expec} yields the result.
\end{proof}

By applying propositions \ref{prop:hilbert_term1} and \ref{prop:hilbert_term2} to theorem \ref{thm:bregman_dist_quad} we deduce the convergence rate of the expected error. Notice that theorem \ref{thm:bregman_dist_quad} proposes two alternative bounds, which will lead to two different estimates: the outcome of the first one will be denoted as \textit{standard bound}, whereas the outcome of the second one as an \textit{alternative bound}. Further comments on the comparison between them is provided in remark \ref{rem:martin}.

\begin{thm}[\textit{Standard} estimate] 
\label{cor:fixed_noise}
\begin{enumerate}[i)]
\item 
Let $f^\dagger \in \widehat{\Omega}\left(\frac 12, L\right)$. Consider ${\mathcal N}(\alpha) = \alpha^{-\frac 1b}$ for $b \geq 1$.  Assume that $\delta>0$ is a constant (independent of $N$).  Then, we have
\begin{equation}
	\label{eq:fixed_noise_rate}
	\E \norm{\fdan- f^\dagger}_X^2 \lesssim L^2 \left(\frac{\delta^2}{L^2 N}\right)^{\frac 13}
	\quad \text{for} \quad \alpha \simeq \left(\frac{\delta^2}{L^2 N}\right)^{\frac 13}.
\end{equation}
\item
Let $f^\dagger \in \widetilde{\Omega}(L)$. Suppose that,  as $N \rightarrow \infty$,  it holds $\frac{\delta^2}{N} \rightarrow 0$ and $N \delta \rightarrow \infty$,  then the rate \eqref{eq:fixed_noise_rate} holds,
whereas when $N\delta$ is bounded the optimal rate is $N^{-1}$ and is achieved by $\alpha \simeq N^{-1}$.
\end{enumerate}
\end{thm}

\begin{proof}
To prove the first statement, we consider the expected value of the first term in \eqref{eq:bregman_dist_quad} and plug in equation \eqref{eq:calR_est1} and proposition \ref{prop:hilbert_term2}, getting
\[
\begin{aligned}
	\E \norm{\fdan- f^\dagger}_X^2 & \lesssim L^2 \alpha {\mathcal B}_N(\alpha) + \frac{\delta^2}{\alpha^2 N}
	\lesssim L^2 \alpha \left(1 + \frac{1}{N\alpha^{\frac{b+1}{b}}} + \frac{1}{N^2\alpha^2} \right) + \frac{\delta^2}{\alpha^2 N} \\
	& \lesssim L^2 \alpha + \frac{1}{\alpha^2 N} \left( \delta^2 + L^2\frac{\alpha}{N} + L^2\alpha^{2-\frac 1b}\right) \lesssim L^2 \alpha + \frac{\delta^2}{\alpha^2 N}.
\end{aligned}
\]
In the last estimate, we have used that $\alpha$ is converging to $0$, hence both $\frac{\alpha}{N}$ and $\alpha^{2-\frac 1b}$ are vanishing and therefore bounded. The optimal choice of $\alpha$ is the one that balances the two remaining terms, hence the one in \eqref{eq:fixed_noise_rate}.

For the second estimate, we use instead the bound \eqref{eq:calR_est1} due to the different source condition, leading to
\begin{equation*}
	\E \norm{\fdan- f^\dagger}_X^2 \lesssim L^2\alpha + \frac{L^2}N + \frac{\delta^2}{\alpha^2 N} \lesssim \alpha \left(
	L^2 +  \frac{L^2}{\alpha N} + \frac{\delta^2}{\alpha^3 N} \right).
\end{equation*}
In order to ensure convergence, $\delta$ need not to be fixed, but still we need to require $\frac{\delta^2}{N} \rightarrow 0$. The optimal choice of $\alpha$ is the one ensuring that the third term in the last summation is bounded and asymptotically equivalent to $L^2$ (hence concluding what is reported in \eqref{eq:fixed_noise_rate}), provided that the second term is vanishing, i.e.,
\[
\frac{1/N}{\left(\delta^2/N\right)^{\frac{1}{3}}} \rightarrow 0 \quad \Rightarrow \quad \frac{1}{\delta N} \rightarrow 0
\]
If this is not the case (i.e., when $\delta N$ is bounded), the optimal $\alpha$ is the one balancing the second term, namely, $\alpha \simeq \frac{1}{N}$.
\end{proof}
Notice that we could extend also the first statement in order to treat the case of non-fixed noise level $\delta$. Nevertheless, for the purpose of this work, statement $i)$ is mainly intended to compare the results carried out via the presented technique with the optimal estimates of the statistical learning literature, in which $\delta$ is typically a constant. 

On the contrary, in statement $ii)$ we admit the possibility for $\delta$ to vary as $N \rightarrow \infty$, which is more common from an inverse problems perspective. To get a more clear interpretation of such statement, suppose that $\delta \simeq N^{-\beta}$: then, \eqref{eq:fixed_noise_rate} shows that the convergence rate is $N^{-\frac{2\beta+ 1}{3}}$ when $-1/2 < \beta \leq 1$ (so, even if the noise is mildly growing), whereas if the noise decay is faster ($\beta > 1$) the convergence rate gets saturated at $N^{-1}$.

\begin{rem}
\label{rem:our_rates_are_suboptimal}
The result in theorem \ref{cor:fixed_noise} is comparable to the rates derived in \cite{blanchard2018optimal}, where it is proven
that the weak or strong minimax optimal rate for $\sqrt{\E \norm{\fdan- f^\dagger}_X^2}$ is given by
\begin{equation*}
	L \left(\frac{\delta^2}{L^2 N}\right)^{\frac{s}{2s +1 + \frac 1b}}
\end{equation*}
for $b>1$ under certain assumptions on the design measure $\mu$ that imply ${\mathcal N}(\alpha) \leq C \alpha^{-\frac 1b}$, i.e., our assumption regarding the effective dimension. Our setup yields asymptotically the same rate only in the limit $b=1$.

Let us briefly explain why such discrepancy emerges: in the Hilbert space setup, the error term can be explicitly solved by
\begin{eqnarray}
	\label{eq:Hilbert_direct}
	f^\dagger - \fdan & = & f^\dagger - (\Bbu + \alpha)^{-1} \Abu^* \gdan\nonumber \\
	& = & f^\dagger - (\Bbu + \alpha)^{-1} (\Bbu f^\dagger + \delta \Abu^* \epsilon_N)\nonumber \\
	& = & \alpha (\Bbu + \alpha)^{-1} f^\dagger - \delta (\Bbu + \alpha)^{-1} \Abu^* \epsilon_N \nonumber \\
	& =: & E_{appr} + E_{sample},
\end{eqnarray}
where terms $E_{appr}$ and $E_{sample}$ are called the approximation and sample error, respectively.

The result in \cite{blanchard2018optimal} is developed by applying the triangle inequality to identity \eqref{eq:Hilbert_direct} and
estimating $\norm{E_{appr}}$ and $\norm{E_{sample}}$ separately. First, the norm of the approximation error bound shown in \cite{blanchard2018optimal} essentially coincides with proposition \ref{prop:hilbert_term1} (the difference being the applicable qualification regime of the regularization scheme, which is more limited here). Second, the sample error bound given in \cite{blanchard2018optimal} is inherently sharper; \cite[Prop. 5.8]{blanchard2018optimal} yields a rate of order ${\mathcal N}(\alpha)/\alpha N = 1/\alpha^{1+\frac 1b} N$ compared to 
\begin{equation*}
	\frac 1{\alpha^2} \E \norm{\Abu^* \epsilon_N}_X^2 \leq \frac{C}{\alpha^2 N}
\end{equation*}
obtained by proposition \ref{prop:hilbert_term2}.
\end{rem}

\begin{rem}[Is it possible to obtain optimal rates?]
As noted in the previous remark our approach developed above can yield suboptimal convergence rates. 
This feature can be traced back to the choice of utilizing operators $\Gamma_1 = \gamma_1 I$ and $\Gamma_2 = \gamma_2 I$
when applying proposition \ref{prop:aux_convex2} in section \ref{sec:p-homog}. Instead, we can set
\begin{equation}
	\label{eq:choice_of_gamma2}
	\Gamma_2 = (\Bbu + \alpha)^{\frac 12}
\end{equation}
and obtain a variance (corresponding to a square of the sample error) term
\begin{equation*}
	\frac 1\alpha R^\star(\delta (\Gamma_2^{-1})^* \Abu^* \epsilon_N) = \frac {\delta^2}{\alpha} \norm{(\Bbu + \alpha)^{-\frac 12} \Abu^* \epsilon_N}_X^2
\end{equation*}
on the right hand side of the standard bound in theorem \ref{thm:bregman_dist_quad}.
With this modification the analysis of the sample error is aligned with \cite{blanchard2018optimal} and we can apply \cite[Prop. 5.2]{blanchard2018optimal} in order to obtain
\begin{equation*}
	\frac 1{\alpha^2} \E \norm{\Abu^* \epsilon_N}_X^2 \, \leq \, \frac 1{\alpha^{1+\frac 1b} N}.
\end{equation*}

Unfortunately, the choice of $\Gamma_2$ in \eqref{eq:choice_of_gamma2} implies that we cannot apply lemma \ref{lem:xuroach} in order to prove theorem \ref{thm:bregman_dist_gen} and we end up with an extra term 
\begin{equation*}
	\frac 1\alpha R(\Gamma_2(f^\dagger - \fdan)) = \frac 1\alpha \norm{(\Bbu + \alpha)^{\frac 12}(f^\dagger - \fdan)}_X^2
\end{equation*}
on the right hand side of the fixed-noise error upper bound in \eqref{eq:bregman_dist_quad}.

Note that while this extra term is unsatisfactory,  in principle,  by applying identity \eqref{eq:Hilbert_direct} and the triangle inequality followed by the technique utilized in \cite{blanchard2018optimal} one could hope to achieve optimal rates.
Clearly, such an argument provides limited insight but demonstrates that the approach could be further developed towards optimality of the rates. It remains part of future work to consider implications of general operators $\Gamma_1$ and $\Gamma_2$ in an arbitrary $p$-homogeneous case.
\end{rem}

By considering the second estimate proposed in \eqref{eq:bregman_dist_quad}, we can derive an alternative bound for the error. The slightly modified source condition enables the use of a stronger estimate in proposition \ref{prop:hilbert_term1}. 
However, as we will see, the obtained rate is weaker than in theorem \ref{thm:bregman_dist_quad}, which is discussed below.

\begin{prop}[\textit{Alternative} estimate]
\label{cor:vanishing_noise}
Let $f^\dagger \in \widetilde\Omega(L)$.  Suppose that, as $N \rightarrow \infty$,  it holds $\delta \rightarrow 0$ and $N \delta \rightarrow \infty$,  then
\begin{equation}
	\label{eq:vanishing_noise_rate}
	\E \norm{\fdan- f^\dagger}_X^2 \lesssim L \delta \quad \text{for} \quad \alpha \simeq \frac{\delta}{L};
\end{equation}
whereas if $N\delta$ is bounded the optimal rate is $N^{-1}$ and is achieved by $\alpha \simeq N^{-1}$.
\end{prop}

\begin{proof}
Applying inequality \eqref{eq:calR_est2} to the second estimate in \eqref{eq:bregman_dist_quad} we obtain
\begin{equation*}
	\E \norm{\fdan- f^\dagger}_X^2 \lesssim L^2\left(\alpha  + \frac 1N\right) + \frac{\delta^2}{\alpha} \lesssim \alpha \left( L^2 +  \frac{L^2}{\alpha N} +  \frac{\delta^2}{\alpha^2}\right).
\end{equation*}
As in the proof of theorem \ref{cor:fixed_noise}, the optimal rate is obtained by selecting $\alpha$ so that the third term in the summation is bounded and asymptotically equivalent to $L^2$, i.e.,  $\alpha \simeq \frac{\delta}{L}$,  provided that the second one is vanishing ($\frac{1}{N\delta}\rightarrow 0$); otherwise, saturation occurs on the rate $N^{-1}$.
\end{proof}

We immediately notice that the obtained rate is weaker than theorem \ref{thm:bregman_dist_quad}.  This can be seen by an application of Young's inequality
\begin{equation*}
	\left(\frac{\delta^2}{N}\right)^{\frac 13} \lesssim \delta + \frac 1N \leq \max\left\{\delta, \frac 1N\right\},
\end{equation*}
where H\"older conjugates $3/2$ and $3$ were applied.
What we observe is that the saturation point $1/N$ in both convergence rates is due to the Monte Carlo-type estimate in
proposition \ref{prop:hilbert_term1}. If a faster concentration bound of type
\begin{equation*}
		\E \calR(\beta, \bu; f^\dagger) \leq C L^2\left(\beta  + \frac 1{N^\rho}\right)
\end{equation*}
for $\rho>1$ could be derived, it can be seen that the alternative scheme is preferable in small noise regime such that $\delta \lesssim N^{-\rho}$.
This motivates us to consider a general concentration bound for $\calR$ in the $p$-homogeneous case in the next section.

\subsection{Convergence rates for $1<p\leq 2$}

In this subsection we derive general convergence rate results for the $p$-homogeneous case with $1<p\leq 2$.
We develop the results under assumptions on the concentration of expectations of the random terms appearing in upper bounds of theorem \ref{thm:bregman_dist_gen} that generalize the usual bias and variance terms. Such conditions are then proved for specific cases in later sections.

\begin{thm}[$p$-homogeneous case, standard estimate]
\label{thm:general_rate}
Consider the $p$-homogeneous regularization functional defined in \eqref{eq:R=phomog} applied to the direct problem introduced in equations \eqref{eq:IP} and \eqref{eq:observation}.
Suppose that assumptions \ref{ass:sampling} and \ref{ass:R} are satisfied and that
$R(f^\dagger)\leq L$. Moreover, we assume that there exists a constant $Q>0$ such that
\begin{equation}
\label{eq:ass_on_phomog_rate1}
\E \calR(\beta, \bu; f^\dagger) \leq D_1 \beta + D_2 N^{-Q}
\end{equation}
and 
\begin{equation}
\label{eq:ass_on_phomog_rate2}
\E R^\star(\Abu^* \epsilon_N) \leq D_3 N^{-\frac q2}
\end{equation}
for some fixed values $D_1, D_2, D_3>0$.
Suppose that, as $N \rightarrow \infty$, it holds $\frac{\delta^2}{N} \rightarrow 0$ and $\delta N^{\frac{3Q}{q}-\frac{1}{2}}\rightarrow \infty$: then
\begin{equation}
	\label{eq:param_choice_p_fixed}
	\E D_R(\fdan, f^\dagger) \lesssim \left(D_1^2 D_3^{\frac{2}{q}} L^{\frac{q-2}{q}}\right)^{\frac{1}{3}} \left( \frac{\delta^2}{N}\right)^{\frac{1}{3}} \quad \text{for} \quad \alpha \simeq \left(\frac{D_3^{\frac{2}{q}} L^{\frac{q-2}{q}}}{D_1}\right)^{\frac{1}{3}} \left( \frac{\delta^2}{N}\right)^{\frac{1}{3}};
\end{equation}
whereas if $\delta N^{\frac{3Q}{q}-\frac{1}{2}}$ is bounded the optimal rate is $N^{-\frac{2Q}{q}}$ and is achieved by $\alpha \simeq N^{-\frac{2Q}{q}}$.
\end{thm}

\begin{proof}
Let us first estimate
\begin{equation*}
	\left(\gamma_1^p + \frac{\gamma_2^p}{\alpha}\right)^{\frac 2{2-p}}
	\leq C_p \left(\gamma_1^{\frac{2p}{2-p}} + \gamma_2^{\frac{2p}{2-p}} \alpha^{-\frac 2{2-p}}\right)
\end{equation*}
which together with bounds \eqref{eq:ass_on_phomog_rate1} and \eqref{eq:ass_on_phomog_rate2}
yields for the inequality \eqref{eq:bregman_dist_gen2} that
\begin{equation}
	\label{eq:general_rate_aux1}
	\E D_R(\fdan, f^\dagger) \leq C\left( D_1 \alpha + Z_1 \gamma_1^{-q} + Z_2 \gamma_1^{\frac{2p}{2-p}} + Z_3 \gamma_2^{-q} + Z_4 \gamma_2^{\frac{2p}{2-p}}\right),
\end{equation}
where
\begin{eqnarray*}
	Z_1 & = &  D_2 N^{-Q}, \\
	Z_2 & = & D_3 \delta^q \alpha^{-q} N^{-\frac q2} + L,\\
	Z_3 & = & D_3 \delta^q \alpha^{-1} N^{-\frac q2} \quad \text{and}\\
	Z_4 & = & \alpha^{-\frac 2{2-p}}\left(D_3 \delta^q \alpha^{-q} N^{-\frac q2} + L\right).
\end{eqnarray*}
In order to optimize $\gamma_1$ and $\gamma_2$, we record the following calculation: a function $\varphi(\gamma) = a \gamma^{-q} + b \gamma^{\frac{2p}{2-p}}$ is minimized at $\gamma_* = \left(\frac a b\right)^{\tilde p}$, where $\tilde p = \frac{(2-p)(p-1)}{p^2}$.
At the minimizer the function $\varphi$ obtains value 
\begin{equation}
	\label{eq:simple_min}
	\varphi(\gamma_*) = 2 b \left(\frac a b\right)^{\frac{2(p-1)}p} = 2 a^{\frac 2q} b^{\frac 2p - 1}.
\end{equation}
The optimal choices of $\gamma_1, \gamma_2>0$ in inequality \eqref{eq:general_rate_aux1} is given by
\begin{equation*}
	\gamma_1 = \left(\frac{Z_1}{Z_2}\right)^{\tilde p} \quad \text{and} \quad
	\gamma_2 = \left(\frac{Z_3}{Z_4}\right)^{\tilde p}.
\end{equation*}
This yields a bound
\begin{eqnarray*}
	\E D_R(\fdan, f^\dagger) & \leq & C\left(D_1\alpha + Z_1 \left(\frac{Z_1}{Z_2}\right)^{\frac{-q}{\frac{2p}{2-p} + q}} + Z_3  \left(\frac{Z_3}{Z_4}\right)^{\frac{-q}{\frac{2p}{2-p} + q}}\right) \\
	& = & C\left(D_1\alpha + Z_1^{\frac 2q} Z_2^{\frac 2p - 1} + Z_3^{\frac 2q} Z_4^{\frac 2p - 1}\right).
\end{eqnarray*}
Substituting the expressions of $Z_1,Z_2,Z_3,Z_4$, by direct computations we get
\[
\begin{aligned}
	\E D_R(\fdan, f^\dagger) & \lesssim D_1\alpha + D_3^{\frac{2}{q}} L^{\frac{q-2}{q}}(\delta \alpha^{-1}N^{-1/2})^2 + D_2^{\frac{2}{q}} L^{\frac{q-2}{q}}N^{-\frac{2Q}{q}} \\
	& \quad + D_3 (\delta \alpha^{-1}N^{-1/2})^q + D_2^{\frac{2}{q}}D_3^{\frac{q-2}{q}}(\delta \alpha^{-1}N^{-1/2})^{q-2} N^{-\frac{2Q}{q}},
\end{aligned}
\]
from which we deduce that it is necessary that $(\delta \alpha^{-1} N^{-1/2}) \rightarrow 0$, and therefore we need to require $\delta N^{-1/2} \rightarrow 0$. Moreover, since $q>2$, we can neglect the faster terms and get
\[
\begin{aligned}
	\E D_R(\fdan, f^\dagger) & \lesssim D_1\alpha + D_3^{\frac{2}{q}} L^{\frac{q-2}{q}}(\delta \alpha^{-1}N^{-1/2})^2 + D_2^{\frac{2}{q}} L^{\frac{q-2}{q}}N^{-\frac{2Q}{q}} \\
	& \lesssim \alpha \left(  D_1 + D_3^{\frac{2}{q}} L^{\frac{q-2}{q}}\frac{\delta^2/N}{\alpha^3} + D_2^{\frac{2}{q}} L^{\frac{q-2}{q}}\frac{N^{-\frac{2Q}{q}}}{\alpha} \right). 
\end{aligned}
\]
The optimal rate is obtained by selecting $\alpha$ so that the second term in the summation is bounded and asymptotically equivalent to $D_1$ (which results in the choice described in \eqref{eq:param_choice_p_fixed}), provided that the third term is vanishing, i.e.,
\[
\frac{N^{-\frac{2Q}{q}}}{\alpha} \rightarrow 0 \quad \Rightarrow \quad \frac{N^{-\frac{2Q}{q}}}{\left( \delta^2/N\right)^{\frac{1}{3}}} \rightarrow 0
\quad \Rightarrow \quad \frac{1}{\delta^\frac{2}{3} N^{\frac{2Q}{q} - \frac{1}{3}}} \rightarrow 0,
\]
that is equivalent to requiring  $\delta N^{\frac{3Q}{q}-\frac{1}{2}} \rightarrow \infty$. If instead such term is bounded, the third term dominates and the convergence rate cannot get better than $N^{-\frac{2Q}{q}}$, in accordance with the parameter choice $\alpha \simeq N^{-\frac{2Q}{q}}$.
\end{proof}

\begin{thm}[$p$-homogeneous case, alternative estimate]
\label{thm:general_rate2}
Consider the $p$-homogeneous regularization functional defined in \eqref{eq:R=phomog} applied to the direct problem introduced in equations \eqref{eq:IP} and \eqref{eq:observation}.
Let assumptions \ref{ass:R} and \ref{ass:sampling} be satisfied and $R(f^\dagger)\leq L$. Moreover, assume that the inequality \eqref{eq:ass_on_phomog_rate1} holds for some $Q, D_1, D_2 >0$. Suppose that, as $N \rightarrow \infty$, it holds $\delta \rightarrow 0$ and $\delta N^{\frac{2Q}{q}}\rightarrow \infty$: then
\begin{equation}
	\label{eq:param_choice_p_vanish}
	\E D_R(\fdan, f^\dagger) \lesssim D_1^{\frac{1}{2}} \delta \quad \text{for} \quad \alpha \simeq D_1^{-\frac{1}{2}} \delta;
\end{equation}
whereas if $\delta N^{\frac{2Q}{q}}$ is bounded the optimal rate is $N^{-\frac{2Q}{q}}$ and is achieved by $\alpha \simeq N^{-\frac{2Q}{q}}$.
\end{thm}

\begin{proof}
We have
\begin{equation*}
	\E D_R(\fdan, f^\dagger) \lesssim D_1\alpha + D_2 N^{-Q} \gamma^{-q} + \frac{\delta^2}{\alpha} + \gamma^{\frac{2p}{2-p}} L.
\end{equation*}
It follows from the calculation in equation \eqref{eq:simple_min} that
\begin{equation*}
	\E D_R(\fdan, f^\dagger) \lesssim D_1 \alpha + \frac{\delta^2}{\alpha} + 2 D_2^{\frac 2q}N^{-\frac{2Q}q} L^{\frac 2p -1} \lesssim \alpha \left( D_1 + \frac{\delta^2}{\alpha^2} + D_2^{\frac 2q}L^{\frac{q-2}{q}} \frac{N^{-\frac{2Q}q}}{\alpha} \right),
\end{equation*}
from which we proceed analogously as in the proof of theorem \ref{thm:general_rate}.
\end{proof}

\begin{rem}
\label{rem:martin}
Let us compare the standard and alternative estimates, which read as
\begin{equation*}
	\left(\frac{\delta^2}{N}\right)^{\frac 13} + N^{-\frac{2Q}{q}} \quad \text{vs.} \quad
	\max\left\{\delta, N^{-\frac{2Q}{q}}\right\}.
\end{equation*}
The term $N^{-2Q/q}$ sets the fastest possible rate and in the typical case (see section \ref{sec:BesovConc} for the Besov case) we find $Q = \frac q2$, leading to $1/N$, as in the case $p=2$. In such a regime the alternative estimate does not yield an improvement, only $-2Q/q<-1$ would yield a rate $\delta \simeq N^{-\rho}$, $\rho>1$ for which the alternative estimates are better. In the Hilbert case this would need an improvement of Proposition \ref{prop:hilbert_term1}, from whose proof we see that this is based on the general choice $\overline{w}=S_{\bf u} \tilde{w}$ (being $S_{\bf u}$ the evaluation operator defined as in \eqref{eq:operatorS}). A potential improvement is possible only using a more optimal choice of $\overline{w}$ taking into account specific properties of the operator $A^*$ or a more structured randomness taking again depending on $A^*$ (similar to the proof of Theorem 2.1 in \cite{burger2001error}).
This consideration is outside the focus of this paper, but opens interesting questions for further studies about the optimal balance between approximation and sample errors in problems, where the observational noise is substantially smaller than the inverse of a feasible number of design points.
\end{rem}

\section{Strategies for obtaining concentration rates}
In this section, we prove a concrete convergence rate for the special case when X is a Besov space $B^s_{pp}(\R^d)$ and discuss how to derive convergence rates if a continuous embedding of the Banach space X to some Hilbert space $X_0$ is available.

\subsection{Hoeffding's inequality applied to Besov regularizers}
\label{sec:BesovConc}

Let $X = B^s_{pp}(\R^d) \vcentcolon= B^{s}_p(\R^d)$ be a Besov space \cite{daubechies2004} and 
\begin{equation}
\label{eq:BesovRegu}
	R(f) = \frac 1p \norm{f}^p_{B^{s}_p} := \frac{1}{p} \sum_{\lambda=1}^\infty \coef |\langle f, \psi_\lambda\rangle|^p,
\end{equation}
for some $1< p \, < \, 2$, where
\begin{equation}
	\label{eq:coef}
	\coef = 2^{|\lambda| d \big(p(\frac{s}d + \frac{1}{2}) -1 \big)}.
\end{equation}
Here,  $\psi_\lambda:\R^d\to \R$, with $\lambda = 1,...,\infty$, are suitably regular functions that form an orthonormal wavelet basis for $L^2(\R^d)$ with global indexing $\lambda$. The notation $|\lambda|$ is used to denote the scale of the wavelet basis associated with the index $\lambda$. Notice that when $s = d \left(\frac{1}{p} - \frac{1}{2} \right)$ the Besov norm \eqref{eq:BesovRegu} reduces to the $\ell_p$-norm of the wavelet coefficients.
It is easily seen (e.g., in \cite{schuster2012regularization}) that the convex conjugate of $R$ satisfies
\begin{equation*}
	R^\star(g) = \frac 1q \norm{g}_{B^{-s}_q}^q,
\end{equation*}
where $p$ and $q$ are H\"older conjugates. Clearly, assumption \ref{ass:R} is satisfied by this choice and $R$ is $p$-homogeneous. 

Below, we set $V=\R^d$ and $Y= L^2(U, \R^d)$. 
We assume the following source condition.
\begin{assumption}
\label{ass:source_cond}
Let us define
\begin{equation*}
	\Omega_R(L) := \{f \in X \; | \; R(f) \leq L\}
\end{equation*}
and 
\begin{equation*}
	\Omega_\mu(L) := \{f\in X \; | \; r = \partial R(f) = A_\mu^* w \; \text{for} \; \norm{w}_Z \leq L\}
\end{equation*}
The ground truth $f^\dagger\in X$ satisfies a \emph{classical source condition} if
\begin{equation}
	f^\dagger \in \Omega_R(L_1) \cap \Omega_\mu(L_2)
	\label{eq:SC_Besov}
\end{equation}
for some $0< L_1, L_2 < \infty$.
\end{assumption}

Notice carefully that the domain $U$ does not play a crucial role in the analysis. However, we will employ
the sup-norm on $U$ and, therefore, due to the continuous embedding $Z \subset {\mathcal C}(U; V)$, it is useful to recall that $$\norm{w}_{\infty}=\sup_{u\in U} \norm{w(u)}_V < \norm{w}_Z$$
for any $w\in Z$.

Let us briefly recall the Hoeffding's inequality for sub-Gaussian random variables, i.e., a real-valued random variable $\xi$ is called sub-Gaussian if
\begin{equation*}
	\Prob(|\xi| \geq t) \leq 2 	\exp(-c t^2)
\end{equation*}
for some constant $c>0$. Let us define
\begin{equation*}
	\nsG{\xi} = \inf \left\{c\geq 0 \; \big| \; \E \exp\left(\frac{\xi^2}{c^2}\right) \leq 2\right\}.
\end{equation*}
The Hoeffding's inequality can then be stated as follows.

\begin{prop}[Hoeffding's inequality, \cite{hoeffding1994}]
\label{prop:hoeff}
\begin{itemize}
\item[(1)] Let $\xi_1, ..., \xi_n$ be zero-mean independent random variables bounded on the interval $[a,b]$ containing zero. It holds that
\begin{equation*}
	\Prob\left(\left|\sum_{i=1}^n \xi_i \right|\geq t\right) \leq 2 \exp \left(- \frac{2 t^2}{n(b-a)^2}\right).
\end{equation*}
\item[(2)] Let $\xi_1, ..., \xi_n$ be zero-mean independent sub-Gaussian random variables. It holds that
\begin{equation*}
	\Prob\left(\left|\sum_{i=1}^n \xi_i \right|\geq t\right) \leq 2 \exp \left(- \frac{c t^2}{\sum_{i=1}^n \nsG{\xi_i}^2}\right),
\end{equation*}
where $c > 0$ is an absolute constant.
\end{itemize}
\end{prop}

Let us further make the following technical assumption.
\begin{assumption}
\label{ass:besov_case}
The wavelet basis satisfies
\begin{equation*}
	\sum_{\lambda=1}^\infty \coefb \norm{A \psi_\lambda}_{\infty}^q < \infty,
\end{equation*}
where $\coefb$ is defined according to \eqref{eq:coef}.
\end{assumption}

This requirement can be fulfilled by imposing a sufficiently strong decay of the coefficients $\coefb$, or by some regularity assumptions on the operator $A$. For example, in Section \ref{sub:radon} we show that it holds true for a particular example of a kernel operator $A$, associated with a sufficiently smooth kernel.

\begin{prop}
\label{prop:besov_rate1}
Under the assumptions \ref{ass:source_cond} and \ref{ass:besov_case} it follows that 
\begin{equation*}
	\E \calR(\beta, \bu; f^\dagger)  \leq C_{q,s,d} L_2^q N^{-\frac q2} + L_2^2 \beta 
\end{equation*}
for 
\begin{equation}
	\label{eq:besov_const}
	C_{q,s,d} = C_q \kappa^q \sum_{{\lambda}=1}^\infty \coefb \norm{A \psi_\lambda}_{\infty}^q,
\end{equation}
where $C_q>0$ depends on $q$.
\end{prop}

\begin{proof}

Utilizing the source condition $r^\dagger = A_\mu^* w$ for some $w \in Z$ such that $\norm{w}_Z \leq L_2$, we have
\begin{equation*}
	2 \calR(\beta, \bu; f^\dagger) \leq 
	\frac 1q \norm{(A_\mu^* - \Abu^* \Sbu)w}_{B_q^{-s}}^q + \beta \norm{\Sbu w}_{V_N}^2.
\end{equation*}
The expectation of the second term coincides with $\beta \norm{w}_{Y_\mu}^2$ and can be bounded by $\beta L_2^2$ due to the continuous embedding of $Z$ to $Y_\mu$.
For the first term, we can write
\begin{equation*}
	\langle (A_\mu^*-\Abu^* \Sbu)w, \psi_{\lambda} \rangle  
	= \frac 1N \sum_{n=1}^N \langle (A_\mu^* - A_{u_n}^*S_{u_n}) w, \psi_{\lambda} \rangle =: \frac 1N \sum_{n=1}^N \xi^{\lambda}_n,
\end{equation*}
where we have set
\begin{equation*}
	\xi^{\lambda}_n = \langle (A_\mu^* - A_{u_n}^*S_{u_n}) w, \psi_{\lambda}\rangle = \langle w, A_\mu\psi_{\lambda} \rangle_{Y_\mu} - \langle w(u_n), (A \psi_\lambda)(u_n)\rangle_V
\end{equation*}
It follows that random variables $\xi^{\lambda}_n$ are zero-mean and also i.i.d.~since 
the design points $u_n$ are assumed to be i.i.d..
Furthermore, due to assumption \ref{ass:sampling} and by applying Cauchy--Schwarz inequality we have that
\begin{equation*}
	\langle w(u), (A \psi_\lambda)(u)\rangle_V \leq \norm{w}_Z \norm{A \psi_\lambda}_{\infty}
\end{equation*}
for any $u\in U$. Therefore, for each $\lambda$ the random variables $\xi^{\lambda}_n$, $i=n,...,N$, are bounded uniformly according to
\begin{equation*}
	\langle w, A\psi_{\lambda}\rangle_{Y_\mu} - L_2 \norm{A \psi_\lambda}_{\infty} \leq \xi^{\lambda}_i \leq \langle w, A\psi_{\lambda}\rangle_{Y_\mu} + L_2 \norm{A \psi_\lambda}_{\infty}.
\end{equation*}

Now, by applying the first Hoeffding's inequality it follows that
\begin{eqnarray*}
	\E R^\star((A_\mu^* - \Abu^* \Sbu)w) & = & 
	\sum_{\lambda=1}^\infty \coefb N^{-q}  \E \left|\sum_{n=1}^N \xi^{\lambda}_n\right|^q \\
	& = & \sum_{{\lambda}=1}^\infty \coefb N^{-q} \int_0^\infty t^{q-1} \Prob\left(\left|\sum_{n=1}^N \xi^{\lambda}_n\right| > t\right) dt \\
	& \leq & 2 \sum_{\lambda=1}^\infty \coefb N^{-q} \int_0^\infty t^{q-1} \exp\left( -\frac{t^2}{2N L_2^2 \norm{A \psi_\lambda}_{\infty}^2 } \right) dt	\\
	& = & 2 \sum_{{\lambda}=1}^\infty \coefb N^{-\frac q2} L_2^q \norm{A \psi_\lambda}_{\infty}^q \int_0^\infty s^{q-1}\exp \left(-\frac{1}{2} s^2 \right) ds \\
	& = & C_{q,s,d} L_2 N^{-\frac q2},
\end{eqnarray*}
where we applied a change of variables. Above, $\coefb$ is defined according to \eqref{eq:coef} and the constant $C_{q,s,d}$ is given by \eqref{eq:besov_const} and is bounded due to assumption \ref{ass:besov_case}. Above, $C_q>0$ is only dependent on $q$.

\end{proof}

\begin{prop}
\label{prop:besov_rate2}
Under the above assumptions, it follows that
\begin{equation*}
	\E R^\star (\Abu^* \epsilon_N) \leq \widetilde C_{q,s,d} N^{-q/2},
\end{equation*}
where the constant is given by
\begin{equation}
	\label{eq:besov_const2}
	\widetilde C_{q,s,d} = \widetilde C_q \kappa^q \sum_{{\lambda}=1}^\infty \coefb \norm{A \psi_\lambda}_{\infty}^q 
\end{equation}
and $\widetilde C_q$ is only dependent on $q$.
\end{prop}

\begin{proof}
Similar to the previous proposition, we write
\begin{equation*}
	\langle \Abu^* \epsilon_N, \psi_{\lambda}\rangle  = \frac 1N \sum_{n=1}^N \langle \epsilon_N^n, A_{u_n} \psi_{\lambda}\rangle_V =: \frac 1N \sum_{n=1}^N \tilde \xi_n^{\lambda}.
\end{equation*}
The random variables $\tilde \xi_n^{\lambda}$ are independent and zero-mean, since $\epsilon_N^n$ is zero-mean and independent of $u_n$.
By assumption \ref{ass:sampling} it follows that
\begin{equation*}
	\nsG{\tilde \xi_n^{\lambda}} = \inf \left\{t>0 \; \big| \; \E \exp\left(\frac{(\tilde \xi_n^{\lambda})^2}{t^2}\right) \leq 2 \right\} \\
	\leq \norm{A \psi_\lambda}_{\infty} \nsG{\epsilon_N^n}.
\end{equation*}
As a consequence, by applying the second Hoeffding's inequality we obtain
\begin{eqnarray*}
	\E R^\star(\Abu^* \epsilon_N) & = & \sum_{{\lambda}=1}^\infty \coefb N^{-q}  \E \left|\sum_{n=1}^N \tilde \xi^{\lambda}_n\right|^q \\
		& = & \sum_{{\lambda}=1}^\infty \coefb N^{-q} \int_0^\infty t^{q-1} \Prob\left(\left|\sum_{n=1}^N \tilde \xi^{\lambda}_n\right| > t\right) dt \\
	& \leq & 2 \sum_{{\lambda}=1}^\infty \coefb N^{-q} \int_0^\infty t^{q-1} \exp\left(-\frac{C t^2}{N \norm{A \psi_\lambda}_{\infty}^2} \right) dt \\
	& \leq & 2 \sum_{{\lambda}=1}^\infty \coefb C^{-\frac q2} N^{-\frac q2} \norm{A \psi_\lambda}_{\infty}^q \int_0^\infty s^{q-1}\exp\left(-\frac 12 s^2\right) ds \\
	& = & \widetilde C_{q,s,d} N^{-\frac q2},
\end{eqnarray*}
where the constant $C$ combines the effect of the absolute constant in proposition \ref{prop:hoeff} and the uniform bound on $\norm{\epsilon_N^n}_{sG}$. Moreover, the constant 
 $\widetilde C_{q,s,d}$ is given by \eqref{eq:besov_const2} and is finite due to assumption \ref{ass:besov_case}.
\end{proof}

By applying propositions \ref{prop:besov_rate1} and \ref{prop:besov_rate2} to theorem \ref{thm:bregman_dist_gen} we obtain the following result.

\begin{cor}
\label{cor:BesovEstimates}
Consider the Besov regularizer of \eqref{eq:BesovRegu} applied to the direct problem introduced in equations \eqref{eq:IP} and \eqref{eq:observation}.
Suppose assumptions \ref{ass:sampling}, \ref{ass:source_cond} and \ref{ass:besov_case} hold.
\begin{itemize}
\item[(1)] Standard estimate: if $\frac{\delta^2}{N}\rightarrow 0$ and $N\delta \rightarrow \infty$, 
\begin{equation}
	\E  D_R(\fdan, f^\dagger)  \lesssim L_1^{\frac{2-p}{3p}}L_2^{\frac 43}\left(\frac{\delta^2}{N}\right)^{\frac 13} \quad \text{for} \quad \alpha \simeq L_1^{\frac {2-p}{3p}}L_2^{-\frac 23} \left(\frac{\delta^2}{N}\right)^{\frac 13};
	\label{eq:besov_rate1}
\end{equation}
if instead $N\delta$ is bounded, the optimal rate is $N^{-1}$, associated with the choice $\alpha \simeq N^{-1}$.
\item[(2)] Alternative estimate: if $\delta \rightarrow 0$ and $N\delta \rightarrow \infty$, 
\begin{equation}
\E  D_R(\fdan, f^\dagger) \lesssim L_2 \delta \quad \text{for} \quad 	\alpha \simeq L_2^{-1} \delta;
	\label{eq:besov_rate2}
\end{equation}
if instead $N\delta$ is bounded, the optimal rate is $N^{-1}$, associated with the choice $\alpha \simeq N^{-1}$.
\end{itemize}
\end{cor}

\begin{proof}
Reflecting the results of propositions \ref{prop:besov_rate1} and \ref{prop:besov_rate2} according to the notation of theorems \ref{thm:general_rate} and \ref{thm:general_rate2}, we have
\begin{equation*}
Q = \frac{q}{2}, \quad D_1 = L_2^2, \quad D_2 = C_{q,s,d} L_2^q, \quad D_3 = \widetilde C_{q,s,d} \quad\text{and}\quad L = L_1.
\end{equation*}
Substituting such terms in the statements of theorems \ref{thm:general_rate} and \ref{thm:general_rate2} and without tracking the constants depending only on $p,s,d$, which immediately yields the claim.
\end{proof}

\subsection{Utilizing Hilbert space embeddings}
\label{subsec:embedding}

Let us consider how Hilbert space embeddings of $X$ can be utilized in deriving convergence rates
for the symmetric Bregman distance.
Suppose that the Banach space $X$ can be embedded continuously to some Hilbert space $X_0$ and $\norm{f}_{X_0} \leq \norm{f}_{X}$ for all $f\in X_0$. Due to the embedding property we also have that
\begin{equation}
	\label{eq:embedding_Rstar}
	R^\star(g) \leq \frac 1q \norm{g}_{X_0}^q.
\end{equation}
Below, we identify elements of $X$ and $X_0$ in $X^*$ through the following dependency
\begin{equation*}
	X \subset X_0 = X_0^* \subset X^*.
\end{equation*} 
If the embedding is suitably tight, we can derive useful convergence rates as demonstrated in the following results.

\begin{prop}
Suppose $\sdiff^\dagger = B_\mu^s w \in X_0$ for some $s\in (0,\frac 12)$.
We have
\begin{equation*}
	\calR(\beta, \bu ; f^\dagger) \leq \widehat C_{p,s} \beta^{rs} \norm{w}_{X_0}^r + \frac 1q \norm{(B_\mu^s - \Bbu^s) w)}_{X_0}^q,
\end{equation*}
where the constant $\widehat C_{p,s}$ depends on $p$ and $s$ and 
\begin{equation}
	\label{eq:r_embed}
	r = \frac{p}{p-1 + s(2-p)}.
\end{equation}
\end{prop}

\begin{proof}
The Fenchel dual of $E_{\beta, N}$ is given by
\begin{equation*}
	F_\beta(v; \bu, f^\dagger) = \frac{1}{2\beta} \norm{\Abu v}_{V_N}^2 - \langle \sdiff^\dagger, v\rangle_{X^* \times X} + R(v)
\end{equation*}
and, therefore, the Fenchel duality theorem yields
\begin{equation*}
	\calR(\beta, \bu ; f^\dagger) = -\inf_{v\in X} F_\alpha(v; \bu, f^\dagger).
\end{equation*}
 Due to the embedding property and our assumption on $\sdiff^\dagger$, we have a lower bound
\begin{eqnarray*}
	F_\beta(v; \bu, f^\dagger) & = &  \frac{1}{2\beta} \norm{\Bbu^{\frac 12} v}_{X_0}^2  
	- \langle w, \Bbu^s v\rangle_{X_0} - \langle (B_\mu^s - \Bbu^s) w, v\rangle_{X^* \times X}  + R(v) \\
	& \geq & \frac{1}{2\beta} \norm{\Bbu^{\frac 12} v}_{X_0}^2 - \norm{w}_{X_0} \norm{\Bbu^s v}_{X_0}
	- c_p R^\star((B_\mu^s - \Bbu^s) w) + \frac 1{2p}\norm{v}_{X_0}^p,
\end{eqnarray*}
where we applied the generalized Young's inequality.
Interpolation of the norms yields
\begin{equation*}
	\norm{\Bbu^s v}_{X_0} \leq \norm{\Bbu^{\frac 12} v}_{X_0}^{2s} \norm{v}_{X_0}^{1-2s}
\end{equation*}
and by Young's inequality we obtain
\begin{equation*}
	\norm{w}_{X_0} \norm{\Bbu^{\frac 12} v}_{X_0}^{2s} \norm{v}_{X_0}^{1-2s} 
	\leq \frac{1}{2\beta} \norm{\Bbu^{\frac 12} v}_{X_0}^2 + \frac 1{2p}\norm{v}_{X_0}^p + \widehat C_{p,s} \beta^{rs} \norm{w}_{X_0}^r,
\end{equation*}
where the constant $\widehat C_{p,s}$ depends on $p$ and $s$, and $r$ is defined by \eqref{eq:r_embed}.
In conclusion, we obtain
\begin{equation*}
	F_\beta(v; \bu, f^\dagger) \geq -\widehat C_{p,s} \beta^{rs} \norm{w}_{X_0}^r - c_p R^\star((B_\mu^s - \Bbu^s) w)
\end{equation*}
which yields the claim.
\end{proof}

Since we also have
\begin{equation*}
	R^\star(\Abu^* \epsilon_N) \leq \frac 1q \norm{\Abu^* \epsilon_N}_{X_0}^q,
\end{equation*}
it follows that we can bound $D_R(\fdan, f^\dagger)$ in cases $(i)$ and $(ii)$ of theorem \ref{thm:bregman_dist_gen}
involving the random terms $\norm{B_\mu^s - \Bbu^s}_{{\mathcal L}(X_0)}^q$ and $\norm{\Abu^* \epsilon_N}_{X_0}^q$. Therefore, if the spectral properties of $\Bbu$ are well-understood on $X_0$, propositions 5.2 and 5.5 in \cite{blanchard2018optimal} yield probabilistic bounds on the symmetric Bregman distance, when theorem \ref{thm:bregman_dist_gen} and, in particular, the concentration assumption in inequality \eqref{eq:ass_on_phomog_rate1} is generalized for arbitrary power of $\beta$. This generalization is technical and outside of the scope of this paper.


\section{Random angle X-ray tomography}
\label{sec:xray}

As an application of our theory, we study the case when the operator $A$ is the semidiscrete Radon transform and we perform random sampling of the imaging angles.

\subsection{Semidiscrete Radon transform}
\label{sub:radon}

Before introducing the semidiscrete Radon transform, we start by recalling the classical definition of Radon transform $\mathcal{R}$:
\[
\mathcal{R}f (\theta,s) = \int_{\R} f(s \theta + t \theta^\perp) dt \qquad \theta \in S^1, s \in \R.
\]  
When considering the operator $\mathcal{R}$ acting on a function $f \in X = \{ g \in L^2(\Omega): \supp(g) \subset \overline{\Omega} \}$, with $\Omega \subset \R^2$ bounded, then $\mathcal{R}f$ (the so-called \textit{sinogram}) belongs to the space $L^2([0,2\pi)\times(-\bar{s},\bar{s}))$ for a suitable $\bar{s} > 0$ depending on $\Omega$. 
We aim at defining the sampling operator as a function associating an angle $\theta \in U = [0,2\pi)$ to the sinogram related to that direction, namely, $\mathcal{R}(\theta) = \mathcal{R}(\theta,\cdot) \in L^2(-\bar{s},\bar{s})$. Unfortunately, the sinogram space $L^2([0,2\pi)\times(-\bar{s},\bar{s})) \cong L^2(U; L^2(-\bar{s},\bar{s}))$ does not show sufficient regularity to perform pointwise evaluations with respect to the angles. 
\par
One way to overcome this difficulty is to rely on a semidiscrete version of the Radon transform. In particular, we set the variable $s$ in a discrete space, which corresponds to modeling the X-ray attenuation measurements performed with a finite-accuracy detector, consisting of $\Ndtc$ cells. To this end, we introduce a uniform partition $\{I_1, \ldots, I_{\Ndtc}\}$ of the interval $(-\bar{s}, \bar{s})$, where we denote by $s_j$ the midpoint of each interval $I_j$ and take a continuous positive function $\rho$ of compact support within $(-1,1)$ such that $\int_{-1}^1 \rho = 1$. The semidiscrete Radon transform is a function $A:X \rightarrow L^2([0,2\pi); \R^{\Ndtc})$ such that, for any $f \in X$ and $\theta \in [0,2\pi)$, each component of the vector $A f(\theta) \in \R^{\Ndtc}$ can be written as
\begin{equation}
	\label{eq:Radon_model_eqL}
	[(A f)(\theta)]_j = 
	\int_{I_j} \mathcal{R}f (\theta,s) \rho\left(\frac{s-s_j}{|I_j|}\right) ds = 
	\int_{I_j} \int_{\R} f(s \theta + t \theta^\perp) \rho\left(\frac{s-s_j}{|I_j|}\right) dt ds.
\end{equation}
Notice carefully that, according to the formalism of section \ref{sec:preliminaries}, $X = \{ f \in L^2(\Omega): \supp(f)\subset \overline{\Omega}\}$, $Y = L^2(U;V)$, $U = S^1\cong [0,2\pi)$ and $V = \R^{\Ndtc}$.

We observe that each component of $Af(\theta)$ is a suitable average of $\mathcal{R}f (\theta,s)$ in a subinterval $I_j$.
By the change of variables $x = s \theta + t \theta^\perp$ in equation \eqref{eq:Radon_model_eqL} we observe that
\[
	[(A f)(\theta)]_j = \int_{\R^2} f(x) \rho_j(x,\theta) dx,
\]
being $\rho_j(x,\theta) = \rho\left(\frac{x\cdot\theta-s_j}{|I_j|}\right)$. As a consequence, from the continuity of $\rho$ we can deduce that for any $f \in X$ each component of $Af(\theta)$ is a continuous function of $\theta$, hence we can consider $A: X \rightarrow Z$ being $Z = \mathcal{C}(U;V)$ and the
sampled operator $A_\theta : X \rightarrow V$ is well defined for every $\theta \in U$. Moreover, the following bound holds uniformly in $\theta$:
\[
\norm{A_\theta f}_{V}^2 = \norm{Af(\theta)}_{V}^2 = \sum_{j = 1}^{\Ndtc} \left|\int_{\Omega} f(x)\rho_j(x,\theta)dx\right|^2 \leq \Ndtc |\Omega| \norm{f}_{L^2(\Omega)}^2 \norm{\rho}_{\infty}^2,
\]
and therefore we conclude that $A$ is a bounded operator from $L^2(\Omega)$ to $Z$.
We now verify that the semidiscrete Radon transform satisfies assumption \ref{ass:besov_case} for any choice of wavelet basis $\{\psi_\lambda\}$ and sufficiently regular Besov space $B_{p}^s$.
\begin{prop}
Let $\Omega \subset \R^d$ and $X = B_p^s(\Omega)$, with $1 < p \leq 2$ and $s$ such that $\frac{s}{d} \geq \frac{1}{p}-\frac{1}{2}$. Let $\{\psi_\lambda\}$ be an orthonormal basis of $L^2(\Omega)$. Then, the semidiscrete Radon transform satisfies assumption \ref{ass:besov_case}.
\end{prop}
\begin{proof}
Since, by hypothesis, $q\geq 2$ and $c_{\lambda,q,-s,d} \leq 1$, it is enough to prove that
\[
\sum_{\lambda=1}^\infty \| A \psi_\lambda \|_Z^2 < \infty.
\]
By Sobolev embedding, for any $f \in L^2(\Omega)$,
\[
  \| A f\|_Z^2 =  \| A f \|_{C((0,2\pi);\R^{N_{dtc}})}^2 \leq C_{S} \| A f \|_{H^1((0,2\pi);\R^{N_{dtc}})}^2 = C_{S} \sum_{j=1}^{N_{dtc}} \| h_j \|_{H^1(0,2\pi)}^2,
\]
being $h_j(\theta) = \int_{\Omega} f(x) \rho_j(x,\theta)dx$. Denoting by $\langle \cdot, \cdot \rangle$ the scalar product in $L^2(\Omega)$ have that
\[
\| h_j \|_{H^1(0,2\pi)}^2 = \int_0^{2\pi} \langle f, \rho_j(\cdot,\theta) \rangle^2 d\theta + \int_0^{2\pi} \langle f, \nabla_\theta \rho_j(\cdot,\theta) \rangle^2 d\theta.
\]
Therefore, by Parseval's identity,
\[
\begin{aligned}
\sum_{\lambda=1}^\infty \| A \psi_\lambda \|_Z^2 &\leq C_{S} \sum_{\lambda=1}^\infty \sum_{j=1}^{N_{dtc}} \left( \int_0^{2\pi} \langle \rho_j(\cdot,\theta),\psi_\lambda \rangle^2 d\theta + \int_0^{2\pi} \langle \nabla_\theta \rho_j(\cdot,\theta),\psi_\lambda \rangle^2 d\theta \right) \\
& = C_{S} \sum_{j=1}^{N_{dtc}} \int_0^{2\pi} \left( \sum_{\lambda=1}^\infty \langle \rho_j(\cdot,\theta),\psi_\lambda \rangle^2 + \sum_{\lambda=1}^\infty \langle \nabla_\theta \rho_j(\cdot,\theta),\psi_\lambda \rangle^2 \right) \\
& = C_{S} \sum_{j=1}^{N_{dtc}} \int_0^{2\pi} \left( \| \rho_j(\cdot, \theta) \|_{L^2}^2 + \| \nabla_\theta\rho_j(\cdot, \theta) \|_{L^2}^2\right) \leq C(\Omega, \| \rho\|_{C^1},\{|I_j|\},N_{dtc}).
\end{aligned}
\]
\end{proof}

\subsection{Discretization}
\label{sec:discretization}

In order to perform numerical simulation, we now introduce a fully discretized version of the sampled Radon transform. To this end, we replace the functional space $X$ with $\R^{\Npxl}$ and consider the following discrete model:

\begin{equation}
\gNd = \g_N^\dag + \delta \vec{\epsilon}_N = \Aop_{\thetab} \f^\dag + \delta \vec{\epsilon}_N
\label{eq:DiscrRadonSampl}
\end{equation}
where  $\f^\dag \in \R^{\Npxl}$ denotes the (unknown) discrete and vectorized image, $\Aop_{\thetab} \in \R^{\Ndtc N \times \Npxl}$ represents the sampled version of the Radon operator corresponding to the $N$ randomly sampled angles $\thetab$, $\g_N^\dag \in \R^{\Ndtc N}$ is the subsampled sinogram and $\vec{\epsilon}_N \in \R^{\Ndtc N}$ is the noise. In the implementation, we consider a normal distribution for the noise vector, $\vec{\epsilon} \sim \mathcal{N}(\vec{0},\vec{I}_{\Ndtc N})$, where $\vec{I}_{\Ndtc N}$ is the identity matrix in $\R^{\Ndtc N \times \Ndtc N}$. A practical example is depicted in figure~\ref{fig:SubSino}. 

\begin{figure}
\centering
\begin{tabular}{ccc}
\includegraphics[width=0.3\textwidth]{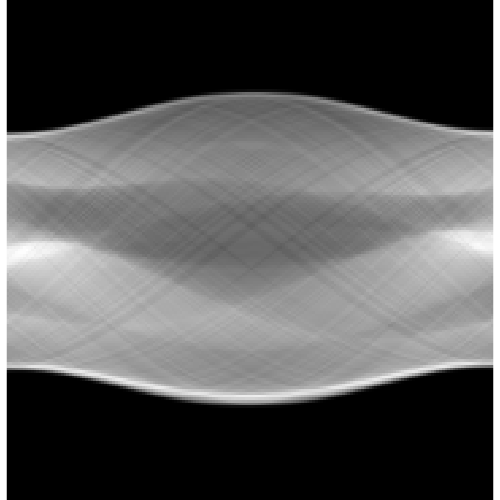} 
	& \includegraphics[width=0.3\textwidth]{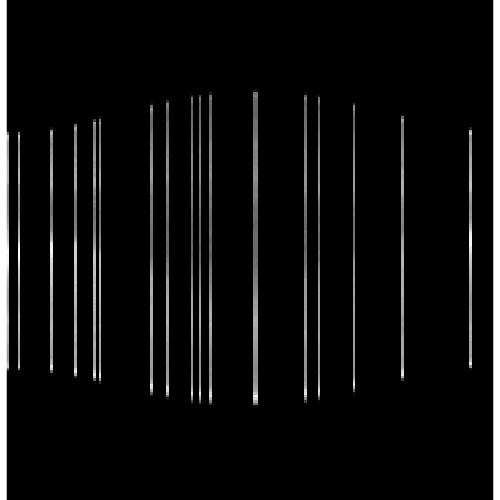} 
	& \includegraphics[width=0.3\textwidth]{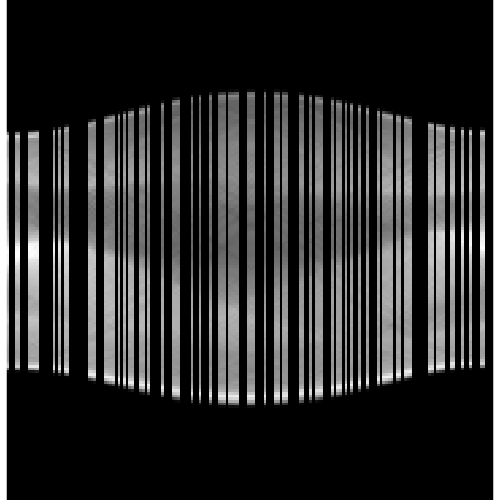} \\
(a) & (b) & (c)
\end{tabular}
\caption{(a) Fully sampled sinogram. (b) Randomly subsampled sinogram with $N = 18$ projection angles. (c) Randomly subsampled sinogram with $N = 81$ projection angles. In each case, the random angles are sampled using Matlab's \texttt{rand}, and therefore are identically distributed and independent.}
\label{fig:SubSino} 
\end{figure}

The numerical experiments are conducted in the presence of the following regularization term:
\begin{equation}
\vec{R}(\f) = \frac{1}{p} \norm{\Wop \f}_p^p
\label{eq:BesovReguDiscr}
\end{equation}
where $1 < p \leq 2$ and $\Wop \in \R^{\Npxl \times \Npxl}$ is an orthogonal matrix. Notice that this expression allows to consider two scenarios of interest:
\begin{enumerate}[i)]
	\item if $p=2$ and $\Wop = \vec{I}_{\Npxl}$, the identity matrix in $\R^{\Npxl \times \Npxl}$, then \eqref{eq:BesovReguDiscr} reduces to the standard Tikhonov regularization, analyzed in subsection~\ref{subsec:Tikhonov};
	\item if $1 < p <2$ and $\Wop$ is the matrix representation of an orthonormal wavelet transform, then \eqref{eq:BesovReguDiscr} represents a Besov norm. In particular, according to \eqref{eq:BesovRegu}, $\vec{R}(\f)$ is equivalent to the $B_p^s(\Omega)$ norm, provided that $s = d\left( \frac{1}{p} -\frac{1}{2} \right)$.
\end{enumerate}
Finally, the discrete counterpart of \eqref{eq:MinFunct} reads as:
\begin{equation}
\faNd = \argmin_{\f \in \R^{\Npxl}} \left\{ \frac{1}{2N} \norm{ \Aop_{\thetab} \f -\gNd }_2^2 + \alpha \vec{R}(\f)  \right\}.   
\label{eq:minsDiscr}
\end{equation}

\subsection{Proximal gradient descent algorithm}
\label{subsec:PGD}
To solve the minimization problem~\eqref{eq:minsDiscr}, we use a proximal gradient descent (PGD) algorithm, adapting the forward-backward algorithm reported in~\cite[Algorithm 10.3]{combettes2011}. In particular, by denoting  $\Phi(\f) = \frac{1}{p} \norm{\f}_p^p$, the $(k+1)$-th iteration of PGD for the minimization of~\eqref{eq:minsDiscr} is given by:
\begin{equation}
\f^{(k+1)} = \Wop^{\text{T}} \prox_{\tau_k \alpha \Phi} \bigg(
	\Wop \Big( \f^{(k)} - \frac{\tau_k}{N} \Aop_{\thetab}^{\text{T}} (\Aop_{\thetab} \f^{(k)} - \gNd) \Big) 
\bigg) 
\label{eq:PGDiter}
\end{equation}
where 
\begin{equation}
\prox_{\alpha\tau\Phi}(\x) = \argmin_{\z \in \R^{\Npxl}} \bigg\{ \frac{1}{2} \norm{\x-\z}_2^2 + \alpha\tau \Phi(\z) \bigg\}
\label{eq:ProxDef}
\end{equation}
is the proximal operator of $\Phi$ and $\tau_k$ is a suitable step length, which we update according to the Barzilai-Borwein rule~\cite{Barzilai88}. 

The expression of $\Phi$ allows to provide a more explicit formula for the associated proximal operator. In particular, since for $1<p\leq 2$ the functional in \eqref{eq:ProxDef} is differential and convex, by first-order optimality condition it holds that if $\bar{\z} \in \R^{\Npxl}$ is such that $ \bar{\z} = \prox_{\alpha\tau\Phi}(\x)$, then $\bar{\z} - \x + \alpha \tau \nabla \Phi(\z) = 0$. Moreover, $\nabla \Phi(\z) = \z^{[p-1]}$, where $\x^{[n]}$ denotes the component-wise signed $n$-th power:
\begin{equation}
(\x^{[n]})_i = \sign (x_i) |x_i|^{n}.
\label{eq:signed_power}
\end{equation}
As a result, the optimality condition satisfied by $\bar{\z}$ reads as follows:
\begin{equation}
\bar{z}_i + \alpha \tau |\bar{z}_i|^{p-1} \sign(\bar{z}_i) - x_i = 0 
\qquad \forall i = 1,\ldots,\Npxl,
\label{eq:ProxDecoupl}
\end{equation}
where all the components are decoupled. Notice that $\sign(x_i) = \sign(\bar{z}_i)$, and so the solution of equation~\eqref{eq:ProxDecoupl} is $\bar{z}_i = \sign(x_i) z_i$ where $z_i$ is the positive solution of
\begin{equation}
z_i + \alpha \tau z_i^{p-1} - |x_i| = 0.
\label{eq:ProxDecPos}
\end{equation}
Therefore, for any choice of $p \in (1,2)$ the proximal $\bar{\z}$ can be efficiently computed by numerically solving $\Npxl$ decoupled equations. Additionally, we remark that for the choice $p=3/2$ and $p=4/3$ (and, in principle $p=5/4$) the solution of equation~\eqref{eq:ProxDecPos} has an explicit, analytic expression given by the formula for the solution of the second, third and fourth degree algebraic equations, respectively. For this reason, without loss of generality, we implement the cases $p=3/2$ and $p=4/3$.

\subsection{Numerical Experiments}
\label{sec:numerical_exp}

In the following, we present and discuss our numerical experiments. Computations were implemented with Matlab R2021a, running on a laptop with 16GB RAM and Apple M1 chip.
\par
The aim is to verify the expected convergence rates proven in theorem~\ref{cor:fixed_noise} for the Tikhonov case and in corollary \ref{cor:BesovEstimates} for the Besov regularization. 
We test inequalities \eqref{eq:fixed_noise_rate} and \eqref{eq:besov_rate1} in the following two scenarios:
\begin{itemize}
\item Fixed noise, i.e., $\delta >0$ constant. Since $\delta N \rightarrow \infty$, we take $\alpha \simeq \delta^{2/3} N^{-1/3}$ and in particular we choose $\delta = c_{\delta}$ and $\alpha = c_{\alpha}/N^{1/3}$;
\item Decreasing noise, e.g., $\delta \simeq N^{-1}$. In this case, the optimal parameter choice is $\alpha \simeq N^{-1}$: therefore, we choose $\delta = c_{\delta} N^{-1}$ and $\alpha = c_{\alpha} N^{-1}$.
\end{itemize}
The positive constants $c_{\delta}$ and $c_{\alpha}$ are specified in the following (see table \ref{tab:BestAlpha} and subsection \ref{subsec:numsetup}).
\par
Notice that in the fixed noise regime, the only valid bounds are the standard estimates \eqref{eq:fixed_noise_rate} and \eqref{eq:besov_rate1}, whereas in the decreasing noise regime the alternative estimates \eqref{eq:vanishing_noise_rate} and \eqref{eq:besov_rate2} are also valid, although since $\delta \simeq N^{-1}$ they actually coincide with the standard ones.

\subsubsection{Implementing the source condition}
\label{subsec:SourceCond}
\begin{figure}[t]
\centering
\begin{tabular}{ccc}
\includegraphics[width=0.3\textwidth]{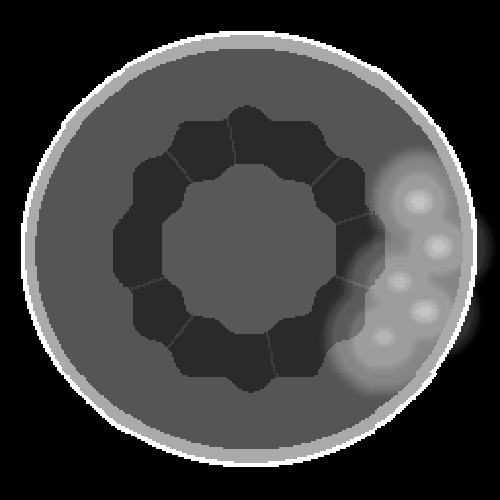} 
	& \includegraphics[width=0.3\textwidth]{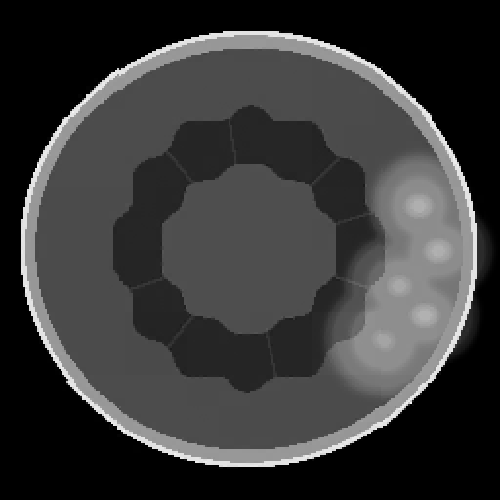} 
	& \includegraphics[width=0.3\textwidth]{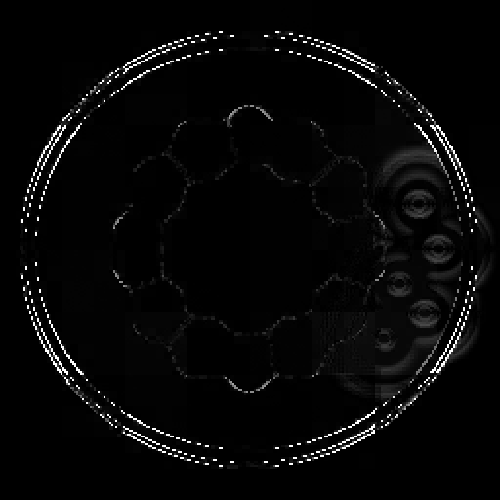} \\
(a) $\f_0$ & (b) $\f^{\dagger}$ & (c) $\f^{\dagger} - \f_0$
\end{tabular}
\caption{(a) Original phantom. (b) Phantom satisfying the source condition in assumption~\ref{ass:source_cond} with $p=3/2$. (c) Difference between (a) and (b), with relative error $5\%$.}
\label{fig:PlantSC}
\end{figure}
All the numerical tests are run on images $\f^{\dagger}$ satisfying the source condition \eqref{eq:SC_Tikhonov} in the Tikhonov case and \eqref{eq:SC_Besov} in the Besov case. In the discrete setting, both \eqref{eq:SC_Tikhonov} and \eqref{eq:SC_Besov} can be formulated as: 
\begin{equation}
\exists  \; \w \in \R^{\Ndtc\Nth} \qquad \text{s.t.} \quad \Wop^{\text{T}} (\Wop \f^{\dagger})^{[p-1]} = \Aop^{\text{T}} \w  
\label{eq:SourceCondDiscr}
\end{equation}
where $\Aop \in \R^{\Ndtc \Nth \times \Npxl}$ with a fixed $\Nth \gg N$ is a matrix representing the Radon transform acting from $\R^{\Npxl}$ to $\R^{\Ndtc \Nth }$, a sufficiently refined discretization of the space $Y$ of full sinograms. As in \eqref{eq:signed_power}, $\f^{[p]}$ is the component-wise signed power.
In practice, a generic phantom of interest $\f_0$ does not necessarily satisfy~\eqref{eq:SourceCondDiscr}. Therefore, 
in order to guarantee that the test images satisfy~\eqref{eq:SourceCondDiscr}, we first determine a vector $\w \in  \R^{\Ndtc\Nth}$ solution of the regularized problem 
\begin{equation}
\w = \argmin_{\widetilde{\w} \in \R^{\Ndtc\Nth}}  \left\{ \frac{1}{2} \norm{\Aop^{\text{T}} \widetilde{\w} - \Wop^{\text{T}} (\Wop \f_0)^{[p-1]}}_2^2 +  \lambda_{SC}   \norm{\widetilde{\w}}_2^2 \right\},
\label{eq:SourceCondTik}
\end{equation}
for a suitable $\lambda_{SC}>0$. Regularization is needed since the inverse problem to determine $\w$ as in~\eqref{eq:SourceCondDiscr} is ill-posed. 
Then, we compute $\f^{\dagger} = \Wop^{\text{T}}(\Wop \Aop^{\text{T}} \w)^{[1/(p-1)]}$: as a result, $\f^{\dagger}$ satisfies the source condition associated with $\w$, and $\norm{\f^{\dagger} - \f_0}_2$ is expected to be small. An example for $p=3/2$ is given in figure~\ref{fig:PlantSC}. Notice that for $p=2$ the source condition~\eqref{eq:SourceCondDiscr} reduces to
\begin{equation}
\exists  \; \w \in \R^{\Ndtc\Nth} \qquad \text{s.t.} \quad  \f^{\dagger} = \Aop^{\text{T}} \w. 
\label{eq:SourceCondp2}
\end{equation}

\subsubsection{Numerical setup} \label{subsec:numsetup}
We use the Plant phantom, available on GitHub~\cite{PlantPhantom2020} (see also figure~\ref{fig:PlantSC}(a)). The size of the phantom is $128 \times 128$, hence $\Npxl = 128^2$. In order to generate a phantom satisfying the source condition, we follow the strategy proposed in subsection~\ref{subsec:SourceCond}, employing the operator $\vec{A}$ with $\Nth = 360$.
The forward operator and its adjoint are implemented using Matlab's \texttt{radon} and \texttt{iradon} routines, with suitable normalization. Reconstructions are computed with $N = 36, 50, 64, \ldots, 162$ in the interval $[0,\pi)$. 
To gain intuition on the subsampling rate associated with this choice, the endpoints $N_0 =36$ and $N_1 = 162$ correspond to $10\%$ and $45\%$ of $\Nth=360$, which is typically used to provide a sufficiently fine discretization of the full sinogram. In each scenario, $N$ random angles are sampled using Matlab's \texttt{rand} (thus ensuring that the random points are stochastically independent).
The Gaussian noise $\vec{\epsilon}_N$ is created by the command \texttt{randn}
and the constant $c_\delta$ appearing in the expression of the noise level $\delta$ is chosen depending on the noise scenario:
\begin{itemize}
\item Fixed noise ($\delta = c_\delta$): $c_{\delta} = 0.01 \norm{\Aop \f^\dag}_{\infty}$;
\item Decreasing noise ($\delta = c_\delta N^{-1}$): $c_{\delta} = 0.02 \, N_0 \norm{\Aop \f^\dag}_{\infty}$. Therefore, $\delta$ ranges between $0.02 \norm{\Aop \f}_{\infty}$ and  $0.02 N_0/N_1 \norm{\Aop \f}_{\infty} \approx 0.005 \norm{\Aop \f}_{\infty}$.
\end{itemize}
Reconstruction are computed using the PGD algorithm as described in subsection~\ref{subsec:PGD}. The regularization parameter $\alpha$ depends on the value of $c_{\alpha}$ which is heuristically determined. 
Optimal values for $c_{\alpha}$ are reported in table~\ref{tab:BestAlpha}.   
\begin{table}
\centering
\begin{tabular}{l|cc}
& fixed noise & reducing noise \\
\hline
$p = 3/2$ &   0.05   &  0.15 \\
$p = 4/3$ &   0.04   &  0.15 \\
$p = 2  $ &   0.13  &   0.16\\
\hline
\end{tabular}
\caption{Optimal values for $c_{\alpha}$.}
\label{tab:BestAlpha}
\end{table}
The expected values appearing in theorem \ref{cor:fixed_noise} and in corollary \ref{cor:BesovEstimates} are approximated by sample averages, computed using $30$ random realizations. This means that, for each number of angles $N$, the reconstruction is performed $30$ times, each time with a different set of $N$ drawn angles and noise vector.

\subsubsection{Numerical results}
In figures \ref{fig:PlantEstimatesNoiseDecreasing} and \ref{fig:PlantEstimatesNoiseFixed} we report the value of the expected Bregman distance $\E D_{\vec{R}}(\f^\delta_{\alpha,N}, \f^\dag)$ as a function of $N$, both in the reducing noise and fixed noise regimes. We compare three different choices of functional $\vec{R}$: $p=3/2$ and $p=4/3$, associated with the choice of the Haar wavelet transform $\vec{W}$ (Besov regularization), and $p=2$ with the identity matrix (Tikhonov regularization). According to corollary \ref{cor:BesovEstimates} and theorem \ref{cor:fixed_noise}, we should expect the same decay of $\E D_{\vec{R}}(\f^\delta_{\alpha,N}, \f^\dag)$, independently of $\vec{R}$: as $N^{-1/3}$ in the fixed noise one, and as $N^{-1}$ in the reducing noise scenario. We can see in figures \ref{fig:PlantEstimatesNoiseFixed} and \ref{fig:PlantEstimatesNoiseDecreasing} that the theoretical behaviour is numerically verified.  
\begin{figure}[t]
\centering
\begin{tabular}{@{}c@{\,}c@{\,}c@{}}
\includegraphics[width=0.32\textwidth]{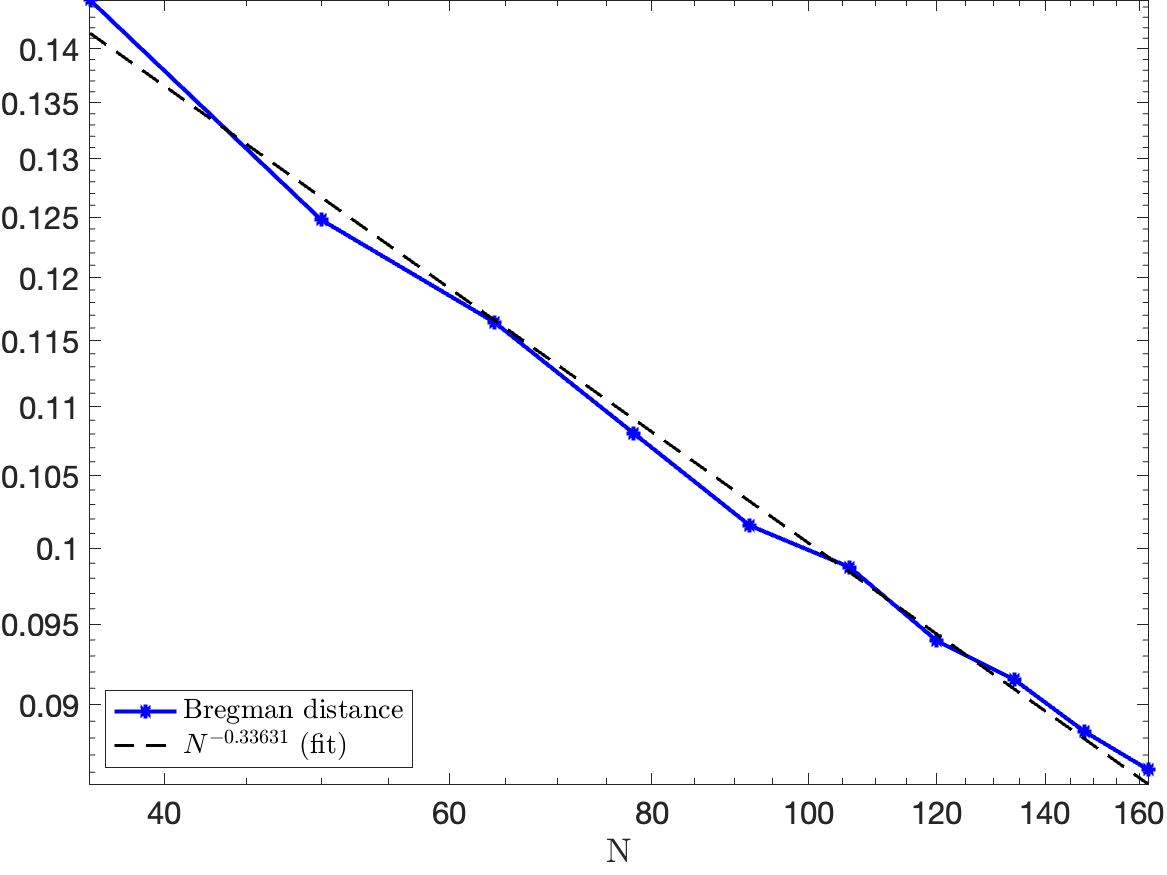} 
	& \includegraphics[width=0.32\textwidth]{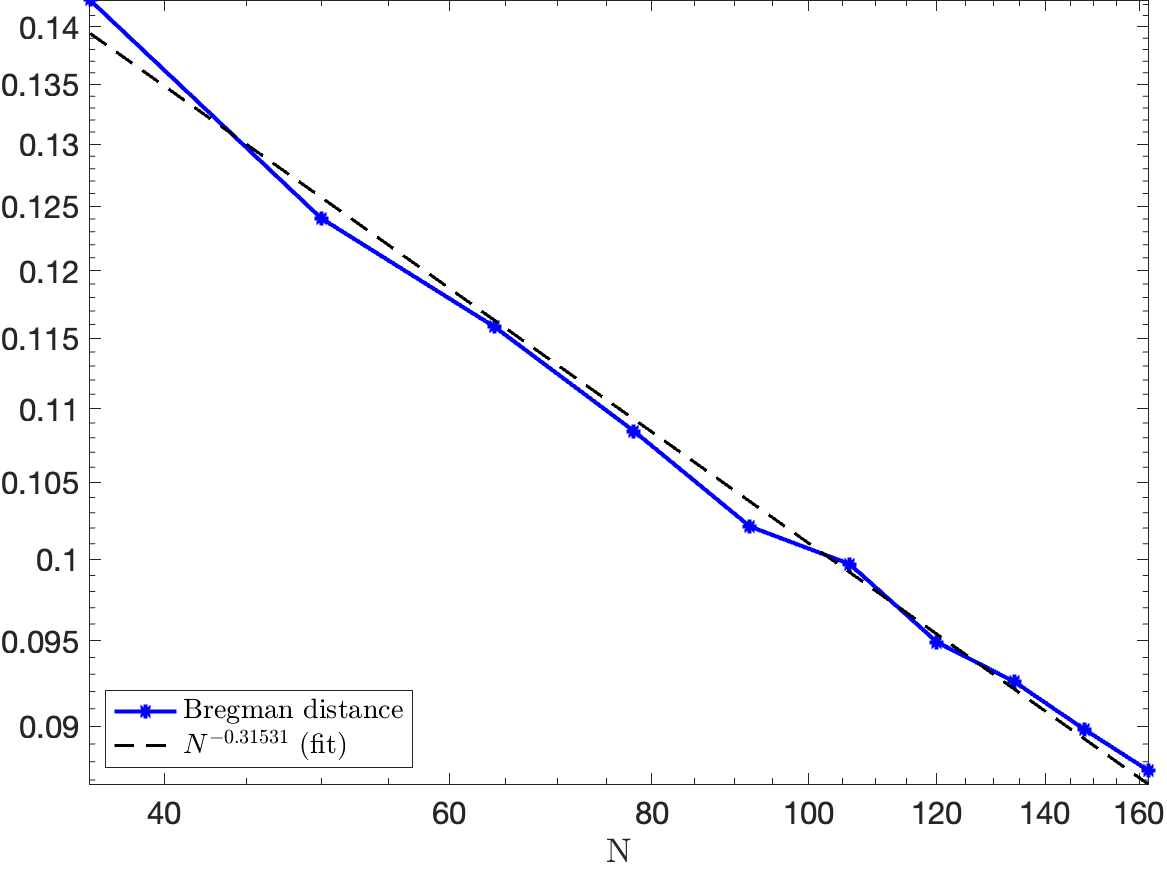} 
	& \includegraphics[width=0.32\textwidth]{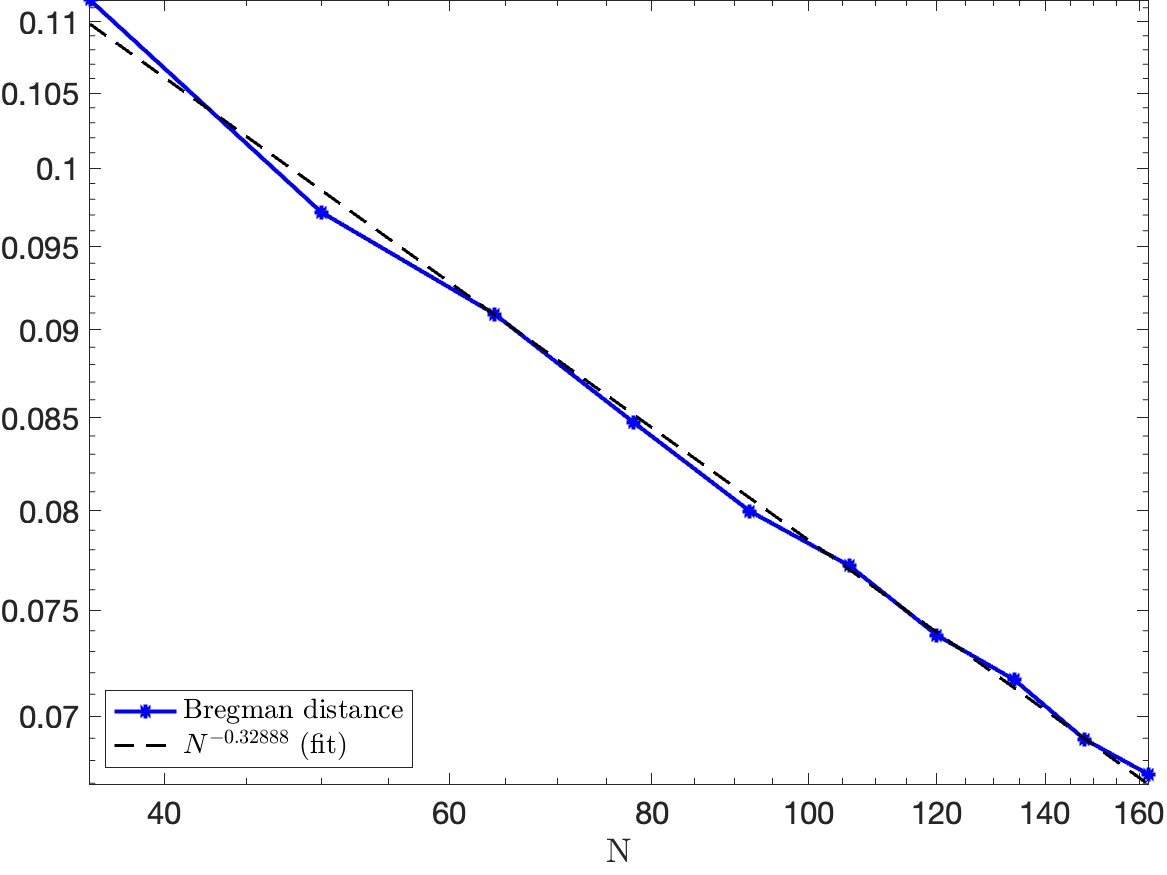} \\
(a) & (b) & (c)
\end{tabular}
\caption{Estimates in the fixed noise case. (a) $p=3/2$ (b) $p=4/3$ (c) $p=2$}
\label{fig:PlantEstimatesNoiseFixed}
\end{figure}
\begin{figure}[t]
\centering
\begin{tabular}{@{}c@{\,}c@{\,}c@{}}
\includegraphics[width=0.32\textwidth]{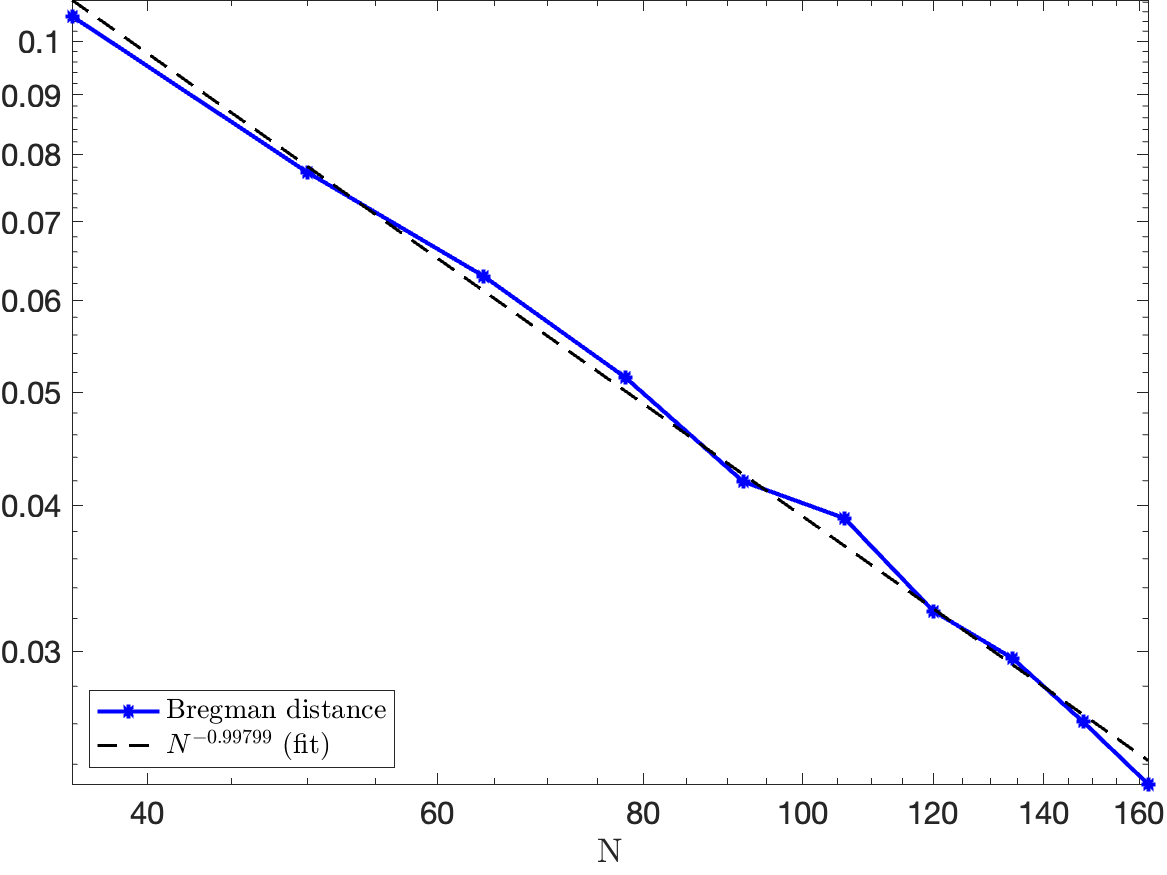} 
	& \includegraphics[width=0.32\textwidth]{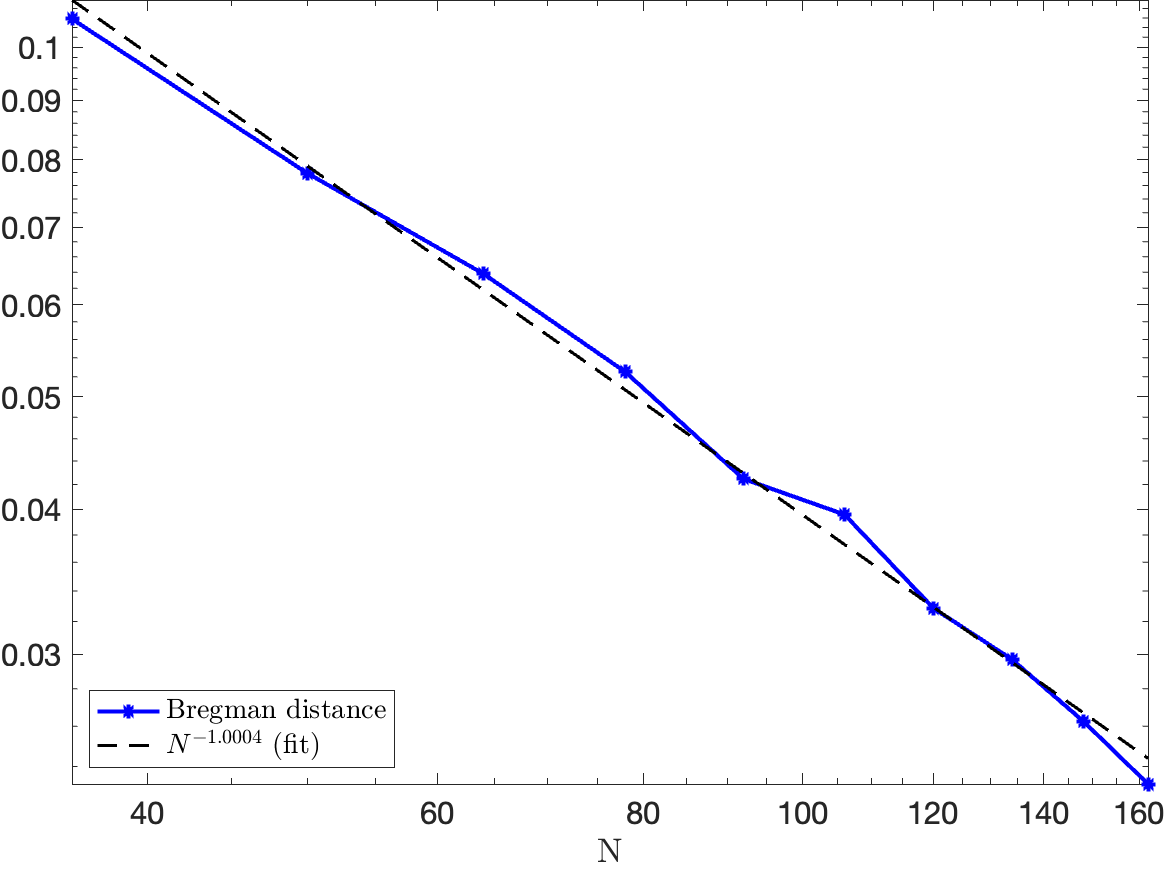} 
	& \includegraphics[width=0.32\textwidth]{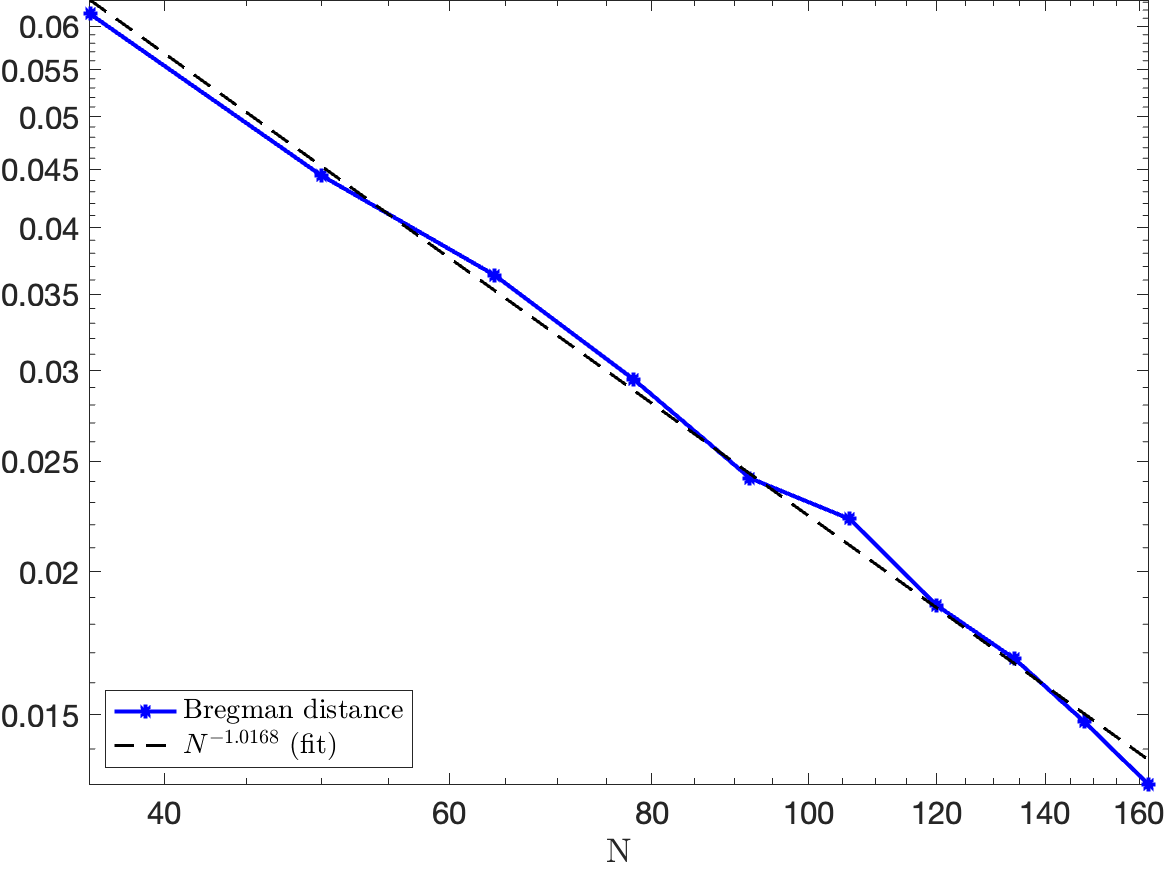} \\
(a) & (b) & (c)
\end{tabular}
\caption{Estimates in the decreasing noise case. (a) $p=3/2$ (b) $p=4/3$ (c) $p=2$}
\label{fig:PlantEstimatesNoiseDecreasing}
\end{figure}

In order to provide a quantitative assessment, we compare the theoretically predicted decay with the experimental one, which is obtained by computing the best monomial approximation $a_{mono}(N) = c N^{\beta}$ of the reported curves. In each plot, the value of the expected Bregman distance is indicated by a blue solid line and its monomial approximation by a black dashed line. We observe a good match with the theoretical previsions, as reviewed in Table \ref{tab:decays}. The results reported in this section allow to conclude that, in the example of the discrete Radon transform, the decay of the expected Bregman distance reported in corollary \ref{cor:BesovEstimates} and theorem \ref{cor:fixed_noise} is verified. We do not attempt to provide an expression for the constants appearing in such inequalities. Moreover, we do not aim at comparing the effectiveness of the three different regularization strategies. For example, it is not worth to compare the values in figure \ref{fig:PlantEstimatesNoiseFixed} (a),(b) and (c) among each other, because the object of the plot, $\E D_{\vec{R}}(\f^\delta_{\alpha,N}, \f^\dag)$, is different in each of them: the Bregman distance clearly depends on $\vec{R}$, but also $\f^\dag$ subtly changes with $\vec{R}$, according to the proposed strategy to impose the source condition to the phantom.
\begin{table}[t]
\centering
\begin{tabular}{c|cccc}
scenario & theoretical & $p = 3/2$ & $p=4/3$ & $p=2$ \\
\hline
reducing noise & $-1$ & $-0.99799$ & $-1.0004$ & $-1.0168$ \\
fixed noise & $-1/3$ & $-0.33631$ & $-0.3153$ & $-0.3289$
\end{tabular}
\caption{Approximate decay $\beta$ of the expected Bregman distance.}
\label{tab:decays}
\end{table}

\section{Conclusions}

In this paper we developed a novel convergence study for a linear forward problem within the statistical inverse learning framework.  We assume a regularization scheme with a general convex $p$-homogeneous penalty functional for $p>1$ and derive concentration rates of the regularized solution to the ground truth measured in the symmetric Bregman distance induced by the penalty functional.  We provide concrete rates for Besov-norm based penalties and observe these rates numerically, for $1 < p \leq 2$, in the case of X-ray tomography with randomly sampled imaging angles.

In the usual framework of statistical inverse learning, the noise level $\delta>0$ is fixed.  Here, we developed estimates also for the asymptotic regime,  where the noise is small with respect to the number of design points, i.e., $\delta \simeq N^{-\rho}$ for some $\rho>1$.  More work is needed to clarify conditions, where such small noise estimates become preferable to the standard framework.  The identity $Q=q/2$ in theorem \ref{cor:vanishing_noise} as observed with the Besov penalties in section \ref{sec:BesovConc} seems natural to the Monte Carlo type approximation error in learning theory.  However,  it is intriguing to understand if and when faster rates with $Q>q/2$ are possible in the small noise regime.

Finally,  the results presented here produce two immediate questions for future studies: first,  it would be valuable to understand whether optimal convergence rates can be achieved with the developed framework.  Second,  arguably the most interesting $p$-homogenous case $p=1$ is not considered here.  Enabling convergence studies for penalties such as Total Variation functional is part of future study.

\section*{Acknowledgments}
TAB was supported by the Academy of Finland through the postdoctoral grant decision number 330522 and is currently supported by the Royal Society through the Newton International Fellowship grant n. NIF\textbackslash R1\textbackslash 201695. TAB and LR acknowledge support by the Academy of Finland through the Finnish Centre of Excellence in Inverse Modelling and Imaging 2018-2025, decision number 312339. 
The work of MB has been supported by ERC via Grant EU FP7 ERC Consolidator Grant 615216 LifeInverse,  by the German Ministry of Science and Technology (BMBF) under grant 05M2020 - Deleto, and by the EU under grant 2020 NoMADS - DLV-777826.
TH was supported by the Academy of Finland through decision number 326961.
LR was supported by the Air Force Office of Scientific Research under award number FA8655-20-1-7027.

\appendix

\section{Technical lemmas}

Let us record here technical lemmas used in section \ref{subsec:Tikhonov}.
The following concentration result was first shown in \cite[Corollary 1]{pinelis1986remarks}.

\begin{prop}
\label{prop:concentration_res}
Let $(Z, {\mathcal B}, \Prob)$ be a probability space and $\xi$ a random variable on $Z$ with values in a real separable Hilbert space ${\mathcal H}$. Assume that there are two positive constants $L$ and $\sigma$ such that for any $m\geq 2$ we have
\begin{equation*}
	\E \norm{\xi - \E \xi}_{{\mathcal H}}^m \leq \frac 12 m! \sigma^2 L^{m-2}.
\end{equation*}
If the sample $z_1, ..., z_N$ drawn i.i.d. from $Z$ according to $\Prob$, then, for any $0<\eta<1$ we have
\begin{equation*}
	\norm{\frac 1N \sum_{j=1}^N \xi(z_j) - \E \xi}_{{\mathcal H}} \leq 2 \log(2\eta^{-1})\left(\frac LN + \frac\sigma{\sqrt N}\right)
\end{equation*}
with probability greater than $1-\eta$.
\end{prop}

\begin{prop}[Cordes inequality \cite{furuta1989norm, fujii1990norm}]
\label{prop:cordes}
Let $T_1, T_2$ be two self-adjoint, positive operators on a Hilbert space. Then for any $s\in [0,1]$ we have
\begin{equation*}
	\norm{T_1^s T_2^s} \leq \norm{T_1 T_2}^s.
\end{equation*}
\end{prop}

The following two results are the basis for estimating expectation of the quadratic loss in section \ref{subsec:Tikhonov}.

\begin{lemma}
\label{lem:app_expec}
Let $X$ be a nonnegative random variable with $\Prob\left(X > Z \log^\gamma \left(\frac k\eta\right)\right) < \eta$ for any $\eta \in (0,1]$.
It follows that
\begin{equation*}
	\E X \leq Z k \gamma \Gamma(\gamma).
\end{equation*}
\end{lemma}

\begin{proof}
The result follows from identity $\E X = \int_0^\infty \Prob(X>t) dt$ and changing variables in the probabilistic bound.
\end{proof}

\begin{prop}\cite[Prop. 1]{guo2017learning}
\label{prop:app_prob_op_bound}
For any $\beta>0$ and $\eta \in (0,1]$ we have
\begin{equation*}
	\norm{(\Bbu+\beta)^{-1}(B_\mu + \beta)} \leq C {\mathcal B}_N(\beta) \log^2 \left(\frac 2\eta\right)
\end{equation*}
for some constant $C>0$ with probability at least $1-\eta$, where ${\mathcal B}_N$ is given by
\begin{equation*}
	{\mathcal B}_N(\beta) = 1 + \left( \frac{2}{N\beta} + \sqrt{\frac{{\mathcal N}(\beta)}{N\beta}}\right)^2.
\end{equation*}
\end{prop}

\begin{proof}
Let us first note that $\norm{AB} = \norm{B A}$ for self-adjoint operators $A$ and $B$ in Hilbert spaces.
Below, we use the decomposition 
\begin{equation*}
	B A^{-1} = (B-A)B^{-1}(B-A)A^{-1} + (B-A)B^{-1} + I
\end{equation*}
for the product. Applying bounds $\norm{(\Bbu+\beta)^{-1}} \leq \frac 1\beta$ and $\norm{(B_\mu+\beta)^{-1/2}} \leq \frac 1{\sqrt{\beta}}$, we obtain 
\begin{equation*}
	\norm{(B_\mu + \beta)(\Bbu+\beta)^{-1}} \leq \norm{(B_\mu + \beta)^{\frac 12}(B_\mu-\Bbu)}^2 \frac 1\beta 
	+ \norm{(B_\mu + \beta)^{\frac 12}(B_\mu-\Bbu)} \frac{1}{\sqrt{\beta}} + 1
\end{equation*}
Now applying the well-established probabilistic estimate \cite[Thm. 4]{caponnetto2007optimal} for $\norm{(B_\mu + \beta)^{\frac 12}(B_\mu-\Bbu)}$ we have for any $\eta \in (0,1]$ that
\begin{equation*}
	\norm{(B_\mu + \beta)(\Bbu+\beta)^{-1}} \leq  \left( \frac{2}{N\beta} + \sqrt{\frac{{\mathcal N}(\beta)}{N\beta}}\right)^2 \log^2\left(\frac 2 \eta\right) 
	+ \left( \frac{2}{N\beta} + \sqrt{\frac{{\mathcal N}(\beta)}{N\beta}}\right) \log\left(\frac 2 \eta\right) +1
\end{equation*}
with probability at least $1-\eta$. The claim follows by simple bounds on the right hand side.

\end{proof}

\bibliographystyle{amsplain}
\bibliography{references}

\end{document}